\documentclass[pdftex,12pt]{article}
\usepackage{authblk}

\usepackage{bm}
\usepackage{amsmath,amssymb,bm}
\usepackage{amssymb}
\usepackage{amsfonts}
\usepackage{amsthm}
\usepackage{comment}
\usepackage{mathrsfs}
\usepackage{graphicx}
\usepackage{color}
\usepackage{amsthm}
\usepackage{tcolorbox}
\usepackage{graphicx}
\usepackage{ascmac}
\usepackage{atbegshi}
\usepackage{lscape}
\usepackage{color}
\usepackage{here}

\usepackage{typearea}
\typearea{17}
\usepackage{cite}

\newtheorem{cl.}{Claim}
\newtheorem{theorem}{Theorem}

\title{Relevance Vector Machine with Weakly Informative Hyperprior and Extended Predictive Information Criterion}
\author[1]{Kazuaki. Murayama}
\author[1]{Shuichi. Kawano}
\affil[1]{Department of Computer and Network Engineering, Graduate School of Informatics and Engineering, The University of Electro-Comunications, 1-5-1 Chofugaoka, Chofu-shi, Tokyo 182-8585, Japan}
\date{\today}

\begin{document}
\maketitle
\begin{abstract}
  In the variational relevance vector machine, the gamma distribution is representative as a hyperprior over the noise precision of automatic relevance determination prior. 
  Instead of the gamma hyperprior, 
  we propose to use the inverse gamma hyperprior with a shape parameter close to zero 
  and a scale parameter not necessary close to zero.  
  This hyperprior is associated with the concept of a weakly informative prior. 
  The effect of this hyperprior is investigated through regression to non-homogeneous data. 
  Because it is difficult to capture the structure of such data with a single kernel function, 
  we apply the multiple kernel method, in which multiple kernel functions with different widths are arranged for input data. 
  We confirm that the degrees of freedom in a model 
  is controlled by adjusting the scale parameter and keeping the shape parameter close to zero. 
  A candidate for selecting the scale parameter is the predictive information criterion. 
  However the estimated model using this criterion seems to cause over-fitting. 
  This is because the multiple kernel method makes the model 
  a situation where the dimension of the model is larger than the data size. 
  To select an appropriate scale parameter even in such a situation, 
  we also propose an extended prediction information criterion. 
  It is confirmed that a multiple kernel relevance vector regression model with good predictive accuracy 
  can be obtained by selecting the scale parameter minimizing extended prediction information criterion. 
\end{abstract}
\section{Introduction\label{sec:1}}
Data with various and complicated nonlinear structures have been available, 
and statistical nonlinear modeling methods have been extensively developed in the fields of machine learning and data science. 
Constructing sparse models with a few covariates and bases is an important research topic, because simple models with sparsity 
often give better prediction than complex models. 
It has been pointed out that the purpose of statistical modeling is not to reproduce the true distribution, 
but to construct a useful model in terms of prediction \cite{akaike1974new,konishi2008information}. 

A representative method for constructing a sparse kernel regression model is the support vector machine (SVM) \cite{vapnik1998statistical,burges1998tutorial,bishop2006pattern}.
Although it had played an important role in the machine learning, 
several problems were pointed out \cite{tipping2000relevance,bishop2000variational,bishop2006pattern}. 
First, the posterior probability for prediction cannot be calculated because a model is constructed using the point estimation. 
Second, the selection of regularization parameters requires intensive calculations such as cross validation. 
To solve these problems, the relevance vector machine (RVM), which constructs a sparse kernel regression model in Bayesian method, was developed \cite{tipping2000relevance,tipping2001sparse}. 
A regression model in RVM is the same as the SVM, i.e., the linear combination of kernel functions. 
The automatic relevance determination (ARD) prior, which induces sparsity \cite{mackay1992evidence,mackay1996bayesian} over weights, is characteristic. 
The posterior probability of many weights will be concentrated around zero, and the sparse model is realized. 
The original method of obtaining the posterior of the weights 
is called the second type maximum likelihood \cite{berger2013statistical} or generalized maximum likelihood \cite{wahba1985comparison}. 
In this method, hyperparameters that maximize the marginal likelihood are estimated \cite{tipping2000relevance,tipping2001sparse}. 
To maximize marginal likelihood at high speed, a fast sequential algorithm was also proposed \cite{tipping2003fast}. 
Another way to obtain the posterior of the weights is the full Bayesian method, applying variational Bayes \cite{jordan1999introduction,jaakkola200110} to the hierarchical Bayes model with hyperpriors \cite{bishop2000variational}. 
We call the full Bayesian approach the VRVM (variational relevance vector machine) 
and distinguish it from the RVM using the second type maximum likelihood. 
A fast sequential algorithms in the VRVM was also developed by Refs. \cite{shutin2011fast1,shutin2011fast2}. 

The VRVM has primarily two hyperparameters. 
The first is the noise precision of likelihood and the second is 
that of ARD prior. 
Our interest of study is directed to what hyperprior over the latter should be used. 
It is important to study this matter 
because it contributes to sparsity in the model. 
So far, the conventional gamma hyperprior has been adopted, almost setting it to be non-informative \cite{tipping2001sparse,bishop2000variational}. 
In this case, the sparsity in estimated model depends on the parameters of the kernel function and 
often causes over-fitting or under-smoothing. 
Although few studies on a hyperprior beyond the gamma hyperprior have been reported, 
we refer to the work by Schmolck and Everson \cite{schmolck2007smooth}. 
This work reported a hyperprior depending on noise precision and associated with wavelet shrinkage. 

As an alternative to the gamma hyperprior, 
we propose to use the inverse gamma hyperprior with a shape parameter close to zero 
and a scale parameter not necessary close to zero. 
This hyperprior is related to a weakly informative prior 
which was proposed in the discussion of prior over variance in hierarchical model \cite{gelman2006prior,gelman2008weakly} and 
a general definition was also given \cite{gelman2013bayesian}. 
In other words, the inverse gamma hyperprior 
combines a non-informative property with some information that contributes to the sparsity. 
To confirm the effect of this hyperprior, 
we perform regression to non-homogeneous data. 
The original RVM and VRVM adopt a single kernel function and let us call these the SK-RVM and SK-VRVM. 
In this case, it is difficult to capture the structure of non-homogeneous data 
which includes smooth and less smooth regions \cite{donoho1995adapting}. 
Therefore, a multiple kernel method \cite{lanckriet2004learning,bach2004multiple,sonnenburg2006large,gonen2011multiple} is applied. 
Such a method has been often used in the RVM and VRVM 
because there is no limit on the type and number of kernel functions \cite{damoulas2008inferring,tzikas2009sparse,psorakis2010multiclass,blekas2014sparse}. 
In this paper, we call these MK-RVM and MK-VRVM. 
It is confirmed that the degrees of freedom in model 
is controlled by adjusting the scale parameter while the shape parameter is fixed on near zero. 

When the inverse gamma hyperprior is adopted, some reasonable selection criterion for selecting the scale parameter is necessary. 
A predictive information criterion ($\mathrm{PIC}$) could be used as such a criterion. 
However, we confirm that the model obtained by PIC tends to cause over-fitting. 
This phenomenon is associated with a situation where the dimension of the model $P$ is larger than the data size $N$, i.e., $P\gg N$ by applying multiple kernel method. 
Similar phenomena was also reported in model selection with the BIC \cite{schwarz1978estimating}. 
The cause of this phenomena seems to be that 
the BIC is constructed by assigning uniform prior probability to each model \cite{schwarz1978estimating}. 
To solve this problem, an extended BIC ($\mathrm{EBIC}$) 
that does not make the prior probability over the models uniform was proposed \cite{chen2008extended}. 
We apply this idea to the traditional $\mathrm{PIC}$, and propose an extended predictive information criterion ($\mathrm{EPIC}$). 
Through regression to non-homogeneous data, 
we confirm that the MK-VRVM with inverse gamma hyperprior whose scale parameter is selected by EPIC 
perform well in terms of predictive accuracy. 

The structure of the paper is as follows: 
In Sec.\ \ref{sec:2}, a general formulation of the nonlinear regression model based on basis expansion and the conventional VRVM are given. 
In Sec.\ \ref{sec:3}, the inverse gamma distribution is introduced as a hyperprior over noise precision of the ARD prior, and 
we explain it with reference to a weakly informative prior. 
In Sec.\ \ref{sec:4}, the EPIC and a method of calculating a bias correction term of it are provided. 
In Sec.\ \ref{sec:5}, the effect of the proposed method is confirmed through numerical experiments. 
In Sec.\ \ref{sec:6}, we discuss our results. 
\section{Nonlinear Regression and Variational Relevance Vector Machine\label{sec:2}}
\subsection{Nonlinear Regression Model with Basis Expansion\label{sec:2-1}}
This subsection describes nonlinear regression model based on basis expansion, where the number of basis functions is $M$. 
Let $\left\{(\bm{x}_n,y_n);n=1,\cdots ,N\right\}$ be the independent observations in terms of a response variable $y$ and 
an explanatory variable $\bm{x}$, where the dimension is assumed to be arbitrary. 
For these observations, a regression model based on basis expansion with basis functions $\left\{\phi_m(\cdot);m=0,\cdots,M-1\right\}$ is 
\begin{equation}
  y_n=w_0+\sum_{m=1}^{M-1}w_m\phi_m(\bm{x}_n)+\varepsilon_n,\qquad n=1,\cdots,N, \label{reg1}
\end{equation}
where $w_m$ are regression weights and errors $\varepsilon_n$ are independently, identically distributed according to $\mathcal{N}(0,\beta^{-1})$. 
For Eq.\ \eqref{reg1}, the weights vector and basis functions vector are defined as 
$\bm{w}=(w_0,w_1,\cdots,w_{M-1})^{\mathrm{T}}$ and $\bm{\phi}(\bm{x}_n)=(\phi_0(\bm{x}_n)=1,\phi_1(\bm{x}_n),\cdots,\phi_{M-1}(\bm{x}_n))^{\mathrm{T}}$, respectively. 
Eq.\ \eqref{reg1} with these vector is reformulated as 
\begin{equation}
  \bm{y}=\bm{\Phi}\bm{w}+\bm{\varepsilon},\label{reg2}
\end{equation} 
where $\bm{y}=(y_1,\cdots,y_n)^{\mathrm{T}}$ is a observations vector, $\bm{\Phi}=(\bm{\phi}(\bm{x}_1),\cdots,\bm{\phi}(\bm{x}_N))^{\mathrm{T}}$ is a design matrix, and 
$\bm{\varepsilon}=(\varepsilon_1,\cdots,\varepsilon_N)^{\mathrm{T}}$ is a error vector. 
In this case, the likelihood is 
\begin{equation}
  p(\bm{y}|\bm{w},\beta)=\mathcal{N}\left(\bm{y}|\bm{\Phi}\bm{w},\beta^{-1}\bm{I}_N\right), \label{likelihood}
\end{equation}
where $\bm{I}_N$ is a $N$-dimensional identity matrix and dependence on the explanatory variables $\bm{x}$ is omitted. 
It is necessary to determine the basis functions in Eq.\ \eqref{reg1} according to the purpose of analysis. 
The representative ones are $B$-spline \cite{de1978practical}, natural cubic splines \cite{green1993nonparametric}, and 
radial basis functions \cite{hastie2009elements}. 
One of the methods estimating $(\bm{w}^{\mathrm{T}},\beta^{-1})^{\mathrm{T}}$ in Eq.\ \eqref{likelihood} 
is maximum likelihood, which sometime cause over-fitting \cite{konishi2008information}. 
To avoid such over-fitting, Lasso \cite{tibshirani1996regression} and Ridge \cite{hoerl1970ridge} were established as regularization methods. 
Another way is Bayesian method, which assuming priors over $(\bm{w}^{\mathrm{T}},\beta^{-1})^{\mathrm{T}}$ and posteriors of them are calculated 
using Bayesian rule. 
\subsection{Variational Relevance Vector Machine\label{sec:2-2}}
In this subsection, variational relevance vector machine (VRVM) \cite{bishop2000variational,tipping2001sparse} is explained and 
we provide a basis to introduce the weakly informative hyperprior. 
The kernel functions are used as basis functions in Eq.\ \eqref{reg1}.  
In our study, we apply the Gaussian kernel functions \cite{simonoff2012smoothing,wand1994kernel} given as 
\begin{equation}
  \phi_m(\bm{x}_n)=K(\bm{x}_n,\bm{x}_m)=\exp\left[-\frac{\|\bm{x}_n-\bm{x}_m\|^2}{2h^2}\right],\qquad n=1,\cdots,N\qquad m=1,\cdots,N, \label{kernelsingle}
\end{equation}
where $h$ is bandwidth parameter with positive value. 
It is assumed that the prior over $\bm{w}$ is the ARD prior with hyperparameter $\bm{\alpha}=\left(\alpha_0,\alpha_1,\cdots,\alpha_{M-1}\right)^{\mathrm{T}}$
 \cite{mackay1992evidence,mackay1996bayesian} defined as 
\begin{equation}
  p(\bm{w}|\bm{\alpha})=\prod_{m=0}^{M-1} \mathcal{N}\left(w_m|0,\alpha _m ^{-1}\right)=\mathcal{N}\left(\bm{0},\bm{A}^{-1}\right), \label{ARD}
\end{equation}
where $\bm{A}=\mathrm{diag}\left(\alpha_0,\cdots,\alpha_{M-1}\right)$. 
Eq.\ \eqref{ARD} enhances the sparsity, i.e., many weights will estimated to be zero. 
The $\bm{x}_n$ corresponding to the remaining non-zero weights are called relevance vectors. 
The gamma hyperprior have been conventionally selected as the hyperpriors over $\bm{\alpha}$ and $\beta$ \cite{tipping2001sparse,bishop2000variational}, i.e., 
\begin{eqnarray}
  p(\bm{\alpha} )&=&\prod_{m=0}^{M-1}\mathrm{Gam}(\alpha _m|a,b)=\prod_{m=0}^{M-1}\frac{1}{\Gamma(a)}b^{a}\alpha_m^{a-1}e^{-b\alpha_m}, \label{alphapriorGam}\\
  p(\beta)&=&\mathrm{Gam}(\beta |c,d)=\frac{1}{\Gamma(c)}d^{c}\beta^{c-1}e^{-d\beta}. \label{betaprior}
\end{eqnarray} 
The previous method applied Eqs.\ \eqref{alphapriorGam} and \eqref{betaprior} setting $a=b=c=d=10^{-6}$ \cite{bishop2000variational}, 
which are close to non-informative hyperprior in terms of Jeffreys prior \cite{jeffreys1946invariant}. 

A method of obtaining posterior distributions with the variational Bayes \cite{jordan1999introduction,jaakkola200110} 
and performing regression with the prediction distribution have been proposed \cite{bishop2000variational}. 
Decomposing the true posterior distribution as $p(\bm{w},\bm{\alpha},\beta|\bm{y})=q(\bm{w})q(\bm{\alpha})q(\bm{\beta})$ 
and minimizing the KL distance between joint distribution of all latent variables and variational posteriors $q(\bm{w})q(\bm{\alpha})q(\bm{\beta})$, 
they are given as 
\begin{eqnarray}
  q(\bm{w})&=&\mathcal{N}(\bm{w}|\tilde{\bm{\mu}},\tilde{\bm{\Sigma}}),\label{qw} \\
  q(\bm{\alpha})&=&\prod_{m=0}^{M-1}\mathrm{Gam}(\alpha_m|\tilde{a}_m,\tilde{b}_m),\label{qalpha} \\ 
  q(\bm{\beta})&=&\mathrm{Gam}(\beta|\tilde{c},\tilde{d}),\label{qbeta}
\end{eqnarray}
where specified parameters are given as 
\begin{eqnarray}
    \tilde{\bm{\Sigma}}&=&\left\{\mathbb{E}{[\bm{A}]}+\mathbb{E}{[\beta]}{\bm{\Phi}^{\mathrm{T}}\bm{\Phi}}\right\}^{-1},\label{varw} \\
    \tilde{\bm{\mu}}&=&\mathbb{E}{[\beta]}\tilde{\bm{\Sigma}}\bm{\Phi}^{\mathrm{T}}\bm{y},\label{meanw}\\
  \tilde{a}_m&=&a+1/2\label{amtil}, \\
  \tilde{b}_m&=&b+\mathbb{E}{\left[w_m^2\right]}/2, \label{bmtil}\\
  \tilde{c}&=&c+N/2,\label{ctil} \\
  \tilde{d}&=&d+\frac{1}{2}\left\{\bm{y}\bm{y}^{\mathrm{T}}-2\mathbb{E}{[\bm{w}]}^{\mathrm{T}}\bm{\Phi}^{\mathrm{T}}\bm{y}+\mathrm{Tr}\bm{\Phi}\mathbb{E}{\left[\bm{w}\bm{w}^{\mathrm{T}}\right]}\bm{\Phi}^{\mathrm{T}}\right\}.\label{dtil}
\end{eqnarray}

Next, the predictive distribution is presented. 
It is given as 
\begin{equation}
  h_{\mathrm{VB}}(\bm{z}|\bm{y})= \iint p(\bm{z}|\bm{w},\beta)q(\bm{w})q(\beta)d\bm{w}d\bm{\beta}, \label{fullyosoku}
\end{equation}
where the true posterior $p(\bm{w},\beta|\bm{y})$ is replaced by Eqs.\ \eqref{qw} and \eqref{qbeta} and 
$\bm{z}$ is future data generated independently on observed $\bm{y}$. 
Analytical integration in terms of $\beta$ in Eq.\ \eqref{fullyosoku} is difficult. 
An approximation that replaces $\beta$ with $\mathbb{E}{[\beta]}$ is applied \cite{bishop2000variational} and Eq.\ \eqref{fullyosoku} is reduced to 
\begin{equation}
  h_{\mathrm{VB}}(\bm{z}|\bm{y})=\int p(\bm{z}|\bm{w},\mathbb{E}{[\beta]})q(\bm{w})d\bm{w}. \label{VByosoku}
\end{equation}
Ref. \cite{bishop2000variational} explained the validity of this approximation as follow: 
when $N$ is large enough, $\beta$ concentrates around $\mathbb{E}{[\beta]}$ because $\mathbb{V}{[\beta]}\sim O(1/N)$. 
The integration in Eq.\ \eqref{VByosoku} can be performed analytically, then we obtain 
\begin{equation}
  h_{\mathrm{VB}}(\bm{z}|\bm{y})=\mathcal{N}(\bm{\mu}_\ast,\bm{\Sigma}_\ast), \label{VByosoku2}
\end{equation}
where specified parameters are given as 
\begin{eqnarray}
  \bm{\mu}_\ast&=&\mathbb{E}{[\beta]}\bm{\Phi}\tilde{\bm{\Sigma}}\bm{\Phi}^{\mathrm{T}}\bm{y},\label{VByosokumean} \\
  \bm{\Sigma}_\ast&=&\mathbb{E}{[\beta]}^{-1}\bm{I}_N+\bm{\Phi}\tilde{\bm{\Sigma}}\bm{\Phi}^{\mathrm{T}}.\label{VByosokuvar}
\end{eqnarray}
The predictive mean Eq.\ \eqref{VByosokumean} is interpreted in the transformation by the hat matrix. 
In this case it is 
\begin{equation}
  \bm{H}=\mathbb{E}{[\beta]}\bm{\Phi}\tilde{\bm{\Sigma}}\bm{\Phi}^{\mathrm{T}}, \label{hatmatrix}
\end{equation}
then we reformulate Eq.\ \eqref{VByosokumean} as $\bm{\mu}_\ast=\bm{H}\bm{y}$. 
Eqs.\ \eqref{varw}-\eqref{dtil} need to be alternately calculated and optimized. 
A variational lower bound is used to determine the convergence of this iterative calculation, which is given as 
\begin{equation}
  \begin{split}
    L=&\mathbb{E}{[\ln p(\bm{y}|\bm{w},\beta)]}+\mathbb{E}{[\ln p(\bm{w}|\bm{\alpha})}]+\mathbb{E}{[\ln p(\beta)]}+\mathbb{E}{[\ln p(\bm{\alpha})]}\\
      &-\mathbb{E}{[\ln q(\bm{w})]}-\mathbb{E}{[\ln q(\beta)]}-\mathbb{E}{[\ln q(\bm{\alpha})]}, \label{LB}
  \end{split}
\end{equation}
where each element is easily evaluated using Eqs.\ \eqref{qw}-\eqref{qbeta} (see Ref. \cite{bishop2000variational} in detail). 
\section{Weakly Informative Hyperprior and Multiple Kernel RVM\label{sec:3}}
\subsection{Inverse Gamma Hyperprior\label{sec:3-1}}
In this subsection, the inverse gamma hyperprior over $\bm{\alpha}$ is proposed instead of the gamma hyperprior. 
Differences between these hyperpriors are described.
Furthermore, an effect of using it is given in light of the concept of the weakly informative prior. 
The inverse gamma hyperprior is 
\begin{equation}
  p(\bm{\alpha})=\prod_{m=0}^{M-1}\mathrm{InGam}(\alpha _m|a,b)=\prod_{m=0}^{M-1}\frac{1}{\Gamma(a)}b^{a}\alpha_m^{-a-1}e^{-b/\alpha_m}, \label{alphaInGam}
\end{equation}
where the shape parameter $a$ is fixed on near zero while the scale parameter $b$ is not necessary close to zero. 
To compare the difference between Eqs.\ \eqref{alphapriorGam} and \eqref{alphaInGam}, 
we consider the weight prior $p(w_m)$ obtained by $p(w_m)= \int p(w_m|\alpha_m)p(\alpha_m)d\alpha_m$ together with $p(\alpha_m)$ focusing on $m$-th component. 
It is calculated as $t$-distribution given as 
\begin{equation}
  p(w_m)=\frac{\Gamma(a+1/2)b^a}{\sqrt{2\pi}\Gamma(a)}\left(b+\frac{w_m^2}{2}\right)^{-(a+1/2)}, \label{pwmGam}
\end{equation}
for gamma hyperprior. For the inverse gamma hyperprior, it is the variance gamma distribution \cite{madan1990variance} given as 
\begin{equation}
  p(w_m)=\frac{2b^a}{\sqrt{2\pi}\Gamma(a)}\left(w_m^2/2b\right)^{-(-a+1/2)/2}K_{(-a+1/2)}(\sqrt{2bw_m^2}), \label{pwmInGam}
\end{equation}
where $K_{(-a+1/2)}(\cdot)$ is the modified Bessel functions of the second kind. 
The $\mathrm{InGam}(\alpha_m|a,b)$ and Eq.\ \eqref{pwmInGam} are plotted in Fig.\ \ref{comp_priors} keeping shape parameter $a=10^{-6}$ and setting scale parameter $b$ several finite value 
together with $\mathrm{Gam}(\alpha_m|a,b)$ and Eq.\ \eqref{pwmGam}. 
We note that the parameter $b$ in gamma hyperprior plays role of rate parameter, not scale parameter. 
\begin{figure}[H]
  \begin{tabular}{cc}
    \hspace{-7pt}
    \begin{minipage}[t]{0.5\hsize}
      \includegraphics[clip,scale=0.5]{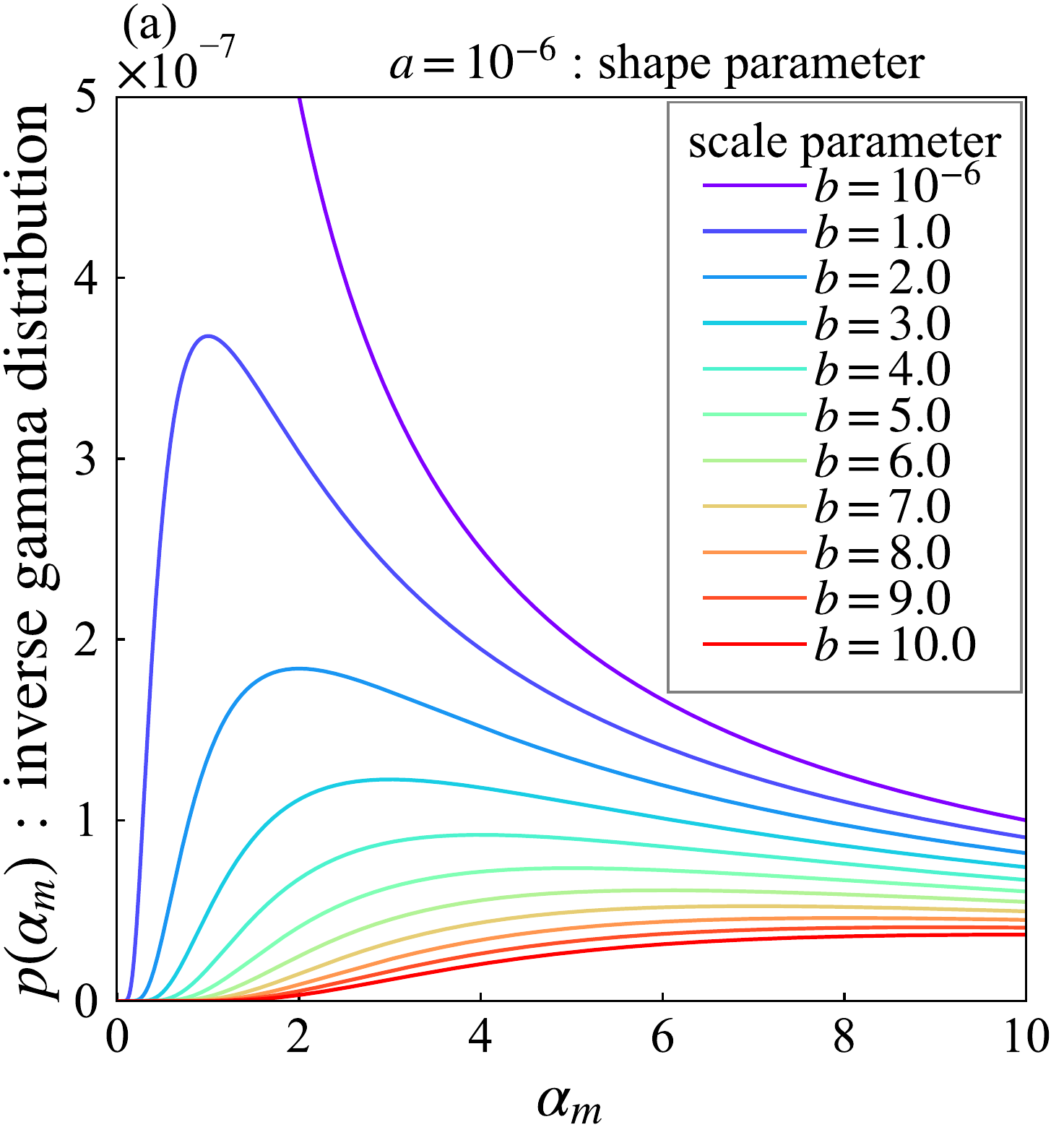}
    \end{minipage} &
    \hspace{-8pt}
    \begin{minipage}[t]{0.5\hsize}
      \includegraphics[clip,scale=0.5]{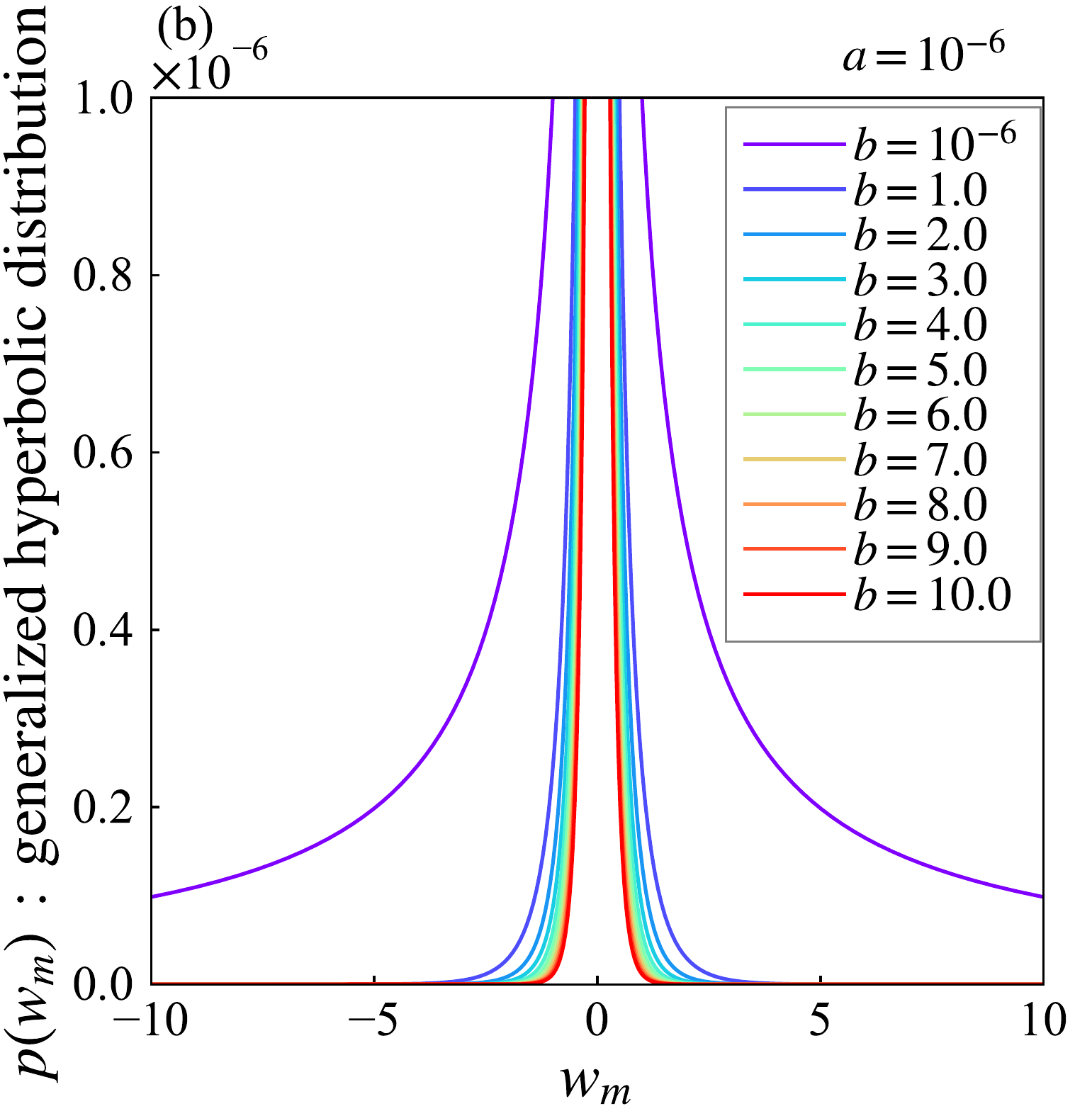}
    \end{minipage}\\
    \hspace{-7pt}
    \begin{minipage}[t]{0.5\hsize}
      \includegraphics[clip,scale=0.5]{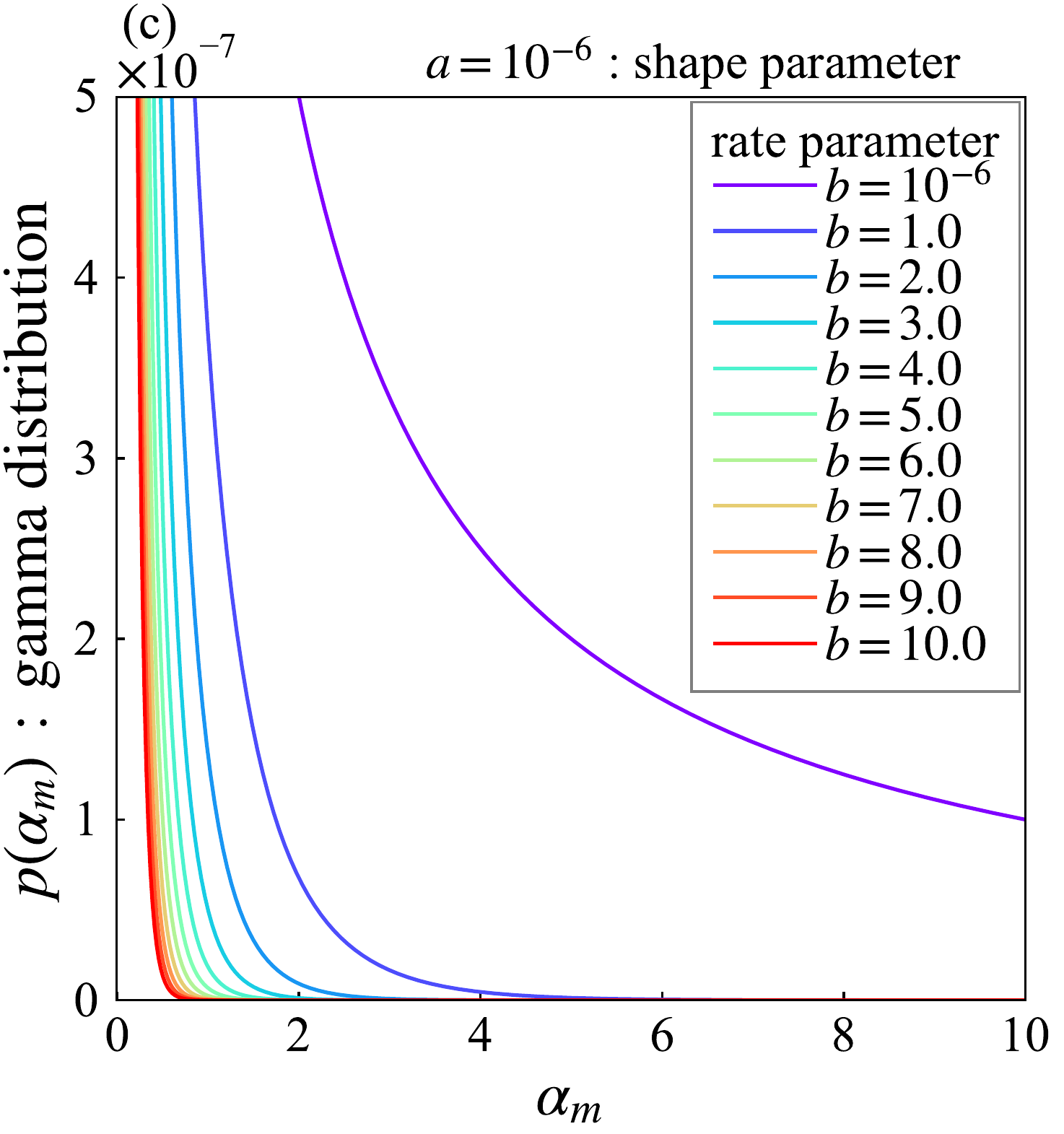}
    \end{minipage} &
    \hspace{-8pt}
    \begin{minipage}[t]{0.5\hsize}
      \includegraphics[clip,scale=0.5]{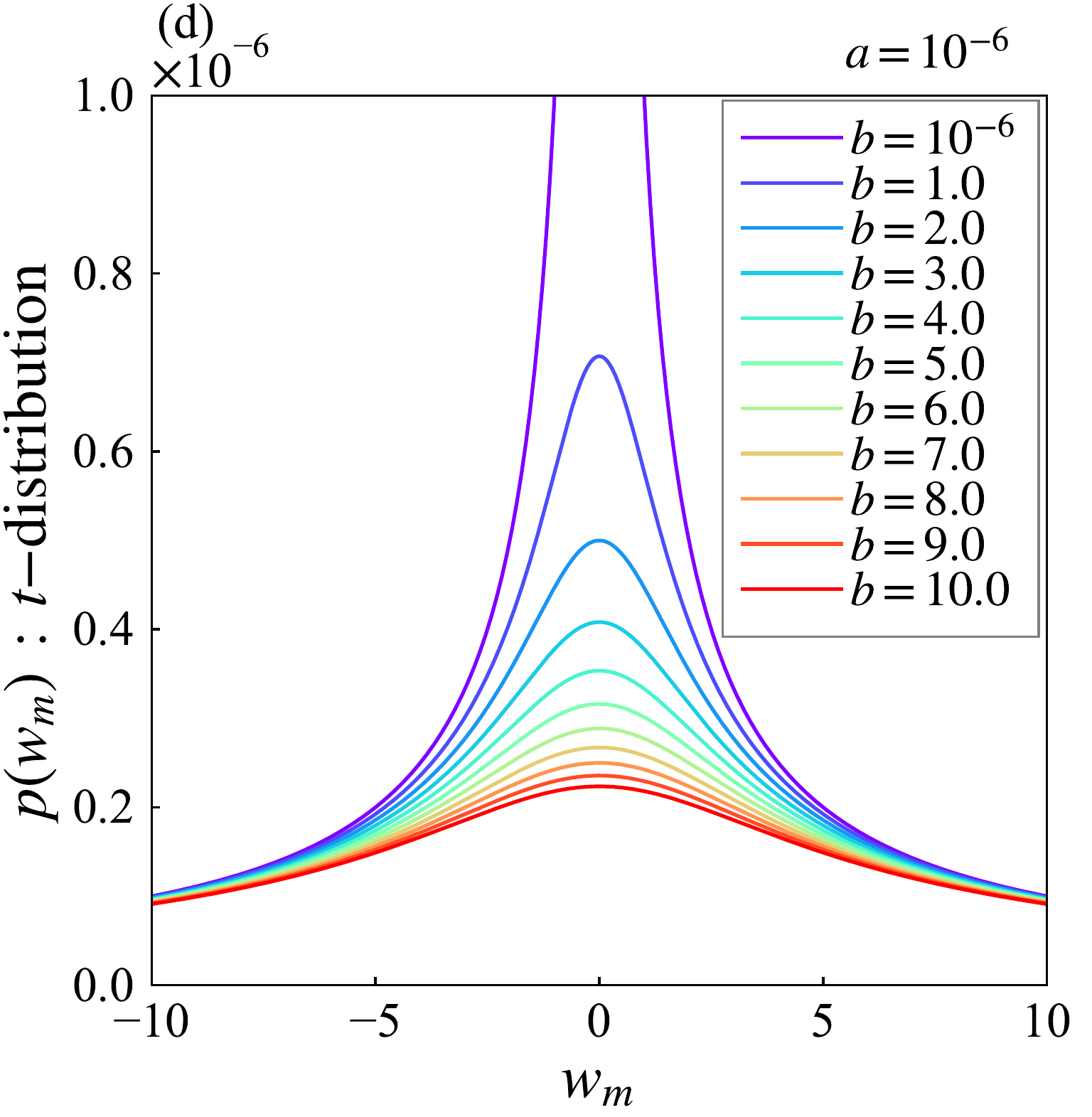}
    \end{minipage}
  \end{tabular}
  \caption{\label{comp_priors} (a) An inverse gamma hyperprior over $\alpha_m$ 
  and (b) corresponding weight prior over $w_m$, generalized hyperbolic distribution. 
  (c) A gamma hyperprior over $\alpha_m$ and (d) corresponding weight prior over $w_m$, $t$-distribution.
  }
\end{figure}
The Jeffreys prior \cite{jeffreys1946invariant}, $p(\alpha_m)\sim 1/\alpha_m$, 
could be realized by setting $(a,b)\sim(0,0)$ for both hyperpriors. 
The difference appears at $(a\sim 0,b\neq 0)$. 
As shown in Fig. \ref{comp_priors}(a), 
the density is shifted to a region where $\alpha_m$ is large by increasing $b$ and fixing $a$ close to zero. 
Such a hyperprior will increase the probability that weights concentrate around zero estimation, 
because $\alpha_m \rightarrow \infty$ corresponds to a zero estimation of $w_m$ (see Eq.\ \eqref{ARD}). 
This property is clearly shown in Fig. \ref{comp_priors}(b), i.e., $p(w_m)$ is concentrated near zero. 
The term $\alpha_m^{-a-1}$ in $\mathrm{InGam}(\alpha_m|a,b)$ is fixed to the Jeffreys prior $\alpha_m^{-1}$. 
By increasing $b$, the influence of the other term $e^{-b/\alpha_m}$ appears, 
including some information regarding sparsity into $\mathrm{InGam}(\alpha_m|a,b)$. 
In other words, $\mathrm{InGam}(\alpha_m|a\sim 0,b \neq 0)\sim \alpha_m^{-1}e^{-b/\alpha_m}$ can be regarded as between non-informative 
and informative. 
This hyperprior with such a effect is consistent with the concept of weakly informative prior, which was defined as 
"{\it it is proper but is set up so that the information it does provide is intentionally weaker than whatever actual knowledge is available}" \cite{gelman2008weakly,gelman2013bayesian}, 
and more simply mentioned as the prior between pure some version non-informative prior and the full informative prior \cite{gelman2013bayesian}. 

In the above definition, the gamma hyperprior with $(a\sim 0,b\neq 0)$ is also the weakly informative prior. 
However, it does not work as hyperprior. 
As shown in Figs.\ \ref{comp_priors}(c) and \ref{comp_priors}(d), 
the density of the region where $\alpha_m$ is large collapses by increasing $b$ 
while that of $w_m$ around zero also collapses. 
In this case, the weights estimated to non-zero will increase, violating the sparsity which is the essential concept of RVM and VRVM. 
\subsection{Variational Posterior corresponding Inverse Gamma Hyperprior\label{sec:3-2}}
In this subsection, the variational Bayes is applied to the hierarchical Bayes model with inverse gamma hyperprior over $\bm{\alpha}$, 
and variational posteriors are introduced. 
The prior over $\bm{w}$ and hyperprior over $\beta$ are the same as in Sec. \ref{sec:2-2}. 
The variational posterior of $\bm{\alpha}$ is generalized inverse Gaussian distribution \cite{barndorff1997normal,bibby2003hyperbolic,kawamura2003characterizations} 
given as 
\begin{equation}
  q(\bm{\alpha})=\prod_{m=0}^{M-1}\mathrm{GIG}(\alpha_m|\tilde{p}_m,\tilde{a}_m,\tilde{b}_m)=\prod_{m=0}^{M-1}\frac{(\tilde{a}_m/\tilde{b}_m)^2}{2K_{\tilde{p}_m}(\sqrt{\tilde{a}_m\tilde{b}_m})}\alpha_m^{\tilde{p}_m-1}\exp{\left(-\frac{\tilde{a}_m\alpha_m+\tilde{b}_m/\alpha_m}{2}\right)}, \label{GIG}
\end{equation}
where $K_p(\cdot)$ is the modified Bessel function of the second kind and specified parameters are given as 
\begin{eqnarray}
  \tilde{p}_m&=&-a+1/2,\label{pmtil} \\
  \tilde{a}_m&=&\mathbb{E}{\left[w_m^2\right]},\label{InGamamtil} \\
  \tilde{b}_m&=&2b.\label{InGambmtil}
\end{eqnarray}
We note that the variational posterior of $\bm{w}$ and that of $\beta$ are the same as Eqs.\ \eqref{qw} and Eq.\ \eqref{qbeta}, respectively. 

The construction of predictive distribution obeys Eq.\ \eqref{VByosoku2}, replacing Eq.\ \eqref{qalpha} with Eq.\ \eqref{GIG}. 
The terms $\mathbb{E}[\ln p(\bm{\alpha})]$ and $-\mathbb{E}[\ln q(\bm{\alpha})]$ in lower bound Eq.\ \eqref{LB} are replaced with 
\begin{equation}
  \mathbb{E}{[\ln p(\bm{\alpha})]}=\sum_{m=0}^{M-1}\left\{a\ln b-(a+1)\mathbb{E}{[\ln\alpha_m]}-b\mathbb{E}{[1/\alpha_m]}-\ln \Gamma(a)\right\}, \label{lnpalpha}\\
\end{equation}
  \begin{eqnarray} 
  -\mathbb{E}{[\ln q(\bm{\alpha})]}&& =-\sum_{m=0}^{M-1}\left\{\tilde{p}_m\ln\sqrt{\frac{\tilde{a}_m}{\tilde{b}_m}}+(\tilde{p}_m-1)\mathbb{E}{[\ln\alpha_m]}\right. \nonumber\\
  &&\left.-\frac{1}{2}\left(\tilde{a}_m\mathbb{E}{[a_m]}+\tilde{b}_m\mathbb{E}{[1/a_m]}\right)-\ln\left(2K_{\tilde{p}_m}\left[\sqrt{\tilde{a}_m\tilde{b}_m}\right]\right) \right\}. \label{lnqalpha}
\end{eqnarray}
The required moments in terms of Eq.\ \eqref{GIG} to calculate Eqs.\ \eqref{varw}, \eqref{lnpalpha}, and \eqref{lnqalpha}
are given as 
\begin{eqnarray}
  \mathbb{E}{[\alpha_m]}&=&\sqrt{\frac{\tilde{b}_m}{\tilde{a}_m}}\frac{K_{\tilde{p}_{m}+1}\left(\sqrt{\tilde{a}_m\tilde{b}_m}\right)}{K_{\tilde{p}_m}\left(\sqrt{\tilde{a}_m\tilde{b}_m}\right)},\label{meanalpha}\\
  \mathbb{E}{[1/\alpha_m]}&=&\sqrt{\frac{\tilde{a}_m}{\tilde{b}_m}}\frac{K_{\tilde{p}_{m}+1}\left(\sqrt{\tilde{a}_m\tilde{b}_m}\right)}{K_{\tilde{p}_m}\left(\sqrt{\tilde{a}_m\tilde{b}_m}\right)}-2\frac{\tilde{p}_m}{\tilde{b}_m},\label{meanalphainv}\\
  \mathbb{E}{[\ln \alpha_m]}&=&\ln\sqrt{\frac{\tilde{b}_m}{\tilde{a}_m}}+\left.\frac{\partial}{\partial p}\ln\left(K_p\left[\sqrt{\tilde{a}_m\tilde{b}_m}\right]\right)\right|_{\tilde{p}_m}.\label{meanalphainv}
\end{eqnarray}

In the case of gamma hyperprior, a fast sequential algorithm \cite{shutin2011fast1,shutin2011fast2} was developed. 
It is desirable to apply such a algorithm to VRVM with inverse gamma hyperprior. 
However, this application seems to be difficult. 
The reasons are discussed in the Appendix \ref{sequentalalgo}. 
\subsection{Multiple Kernel for RVM Regression\label{sec:3-3}}
We confirm the effect of inverse gamma hyperprior over $\bm{\alpha}$ through regression to non-homogeneous data. 
The multiple kernel method \cite{lanckriet2004learning,bach2004multiple,sonnenburg2006large,gonen2011multiple} 
is applied to the VRVM regression, to capture the nonlinear structure of such a data. 
This subsection describes the formulation of the multiple kernel VRVM (MK-VRVM) regression model. 

We use $J$ Gaussian kernel functions with various widths $\{h_j;j=1,\cdots,J\}$ for input data 
instead of single Gaussian kernel function Eq.\ \eqref{kernelsingle}. 
Each function with width $h_j$ is 
\begin{equation}
  K(\bm{x}_n,\bm{x}_m;h_j)=\exp\left[-\frac{\|\bm{x}_n-\bm{x}_m\|^2}{2h_j^2}\right],\qquad n=1,\cdots,N \qquad m=1,\cdots,N. \label{multiplekerenl}
\end{equation}
In this case, regression model Eq.\ \eqref{reg1} is reformulated as 
\begin{equation}
  y_n=w_{0,1}+\sum_{j=1}^J\sum_{m=1}^{N}w_{j,m}K(\bm{x_n},\bm{x_m};h_j)+\varepsilon_n,\qquad n=1,\cdots,N, \label{regmulti}
\end{equation}
where $w_{j,m}$ is $m$-th weight belonging to $j$-th Gaussian kernel with $h_j$ 
and $w_{0,1}$ is bias weight assuming that it is the first weight belonging to the basis function of class $j=0$. 
We denote the weights vector of Gaussian kernel with $h_j$ as $\bm{w}_{j}=\left(w_{j,1},\cdots,w_{j,N}\right)^{\mathrm{T}}$. 
To formulate easily, let us also denote the weight vector of $w_{0,1}$ as $\bm{w}_0=(w_{0,1})^{\mathrm{T}}$ with only one component. 
The overall weights vector $\bm{w}$ is 
\begin{equation}
  \bm{w}=\left(\bm{w}_0^{\mathrm{T}},\bm{w}_1^{\mathrm{T}},\bm{w}_2^{\mathrm{T}},\cdots,\bm{w}_J^{\mathrm{T}}\right)^{\mathrm{T}}. \label{wmulti}
\end{equation}
We also denote the basis functions vector of Gaussian kernel with $h_j$ as \\
$\bm{\phi}_j(\bm{x}_n)=\left(K(\bm{x}_n,\bm{x}_1;h_j),\cdots,K(\bm{x}_n,\bm{x}_N;h_j)\right)^{\mathrm{T}}$ 
together with $\bm{\phi}_0(\bm{x}_n)=(1)^{\mathrm{T}}$ only one component. 
The vector of the overall basis functions is 
\begin{equation}
  \bm{\phi}(\bm{x}_n)=\left(\bm{\phi}_0(\bm{x}_n)^{\mathrm{T}},\bm{\phi}_1(\bm{x}_n)^{\mathrm{T}},\bm{\phi}_2(\bm{x}_n)^{\mathrm{T}},\cdots,\bm{\phi}_J(\bm{x}_n)^{\mathrm{T}}\right)^{\mathrm{T}}. \label{basismulti}
\end{equation}
Eq.\ \eqref{regmulti} is reformulated as Eq.\ \eqref{reg2} using Eqs.\ \eqref{wmulti} and \eqref{basismulti}, 
and follows Sec.\ \ref{sec:2-1}. 
To follows Secs.\ \ref{sec:2-2}, \ref{sec:3-1}, 
and \ref{sec:3-2}, the noise precision of Eq.\ \eqref{wmulti} is extended as 
$\bm{\alpha}=\left(\bm{\alpha}_0^{\mathrm{T}},\bm{\alpha}_1^{\mathrm{T}},\bm{\alpha}_2^{\mathrm{T}},\cdots,\bm{\alpha}_J^{\mathrm{T}}\right)^{\mathrm{T}}$. 
In this case, $\sum_{m=0}^{M-1}$ and $\prod_{m=0}^{M-1}$ are replaced as $\sum_{j=0}^{J}\sum_{m=1}^{\mathrm{dim}(\bm{w}_j)}$ and 
$\prod_{j=0}^{J}\prod_{m=1}^{\mathrm{dim}(\bm{w}_j)}$, respectively, 
where $\mathrm{dim}(\cdot)$ expresses dimensions of vector. 
We note that $\mathrm{dim}(\bm{w}_j)=\mathrm{dim}(\bm{\alpha}_j)=1$ for $j=0$ which corresponds to only one component of bias weight, 
and $\mathrm{dim}(\bm{w}_j)=\mathrm{dim}(\bm{\alpha}_j)=N$ for $\forall j\in\left\{1,2,\cdot,J\right\}$. 
In our study, $J=10$ Gaussian kernels from $h_1=0.005$ to $h_{10}=0.05$ at $0.005$ intervals are applied, which 
are shown in Fig.\ \ref{Figmultiplekerenl} .
\begin{figure}[H]
  \begin{center}
  \includegraphics[scale=0.5,clip]{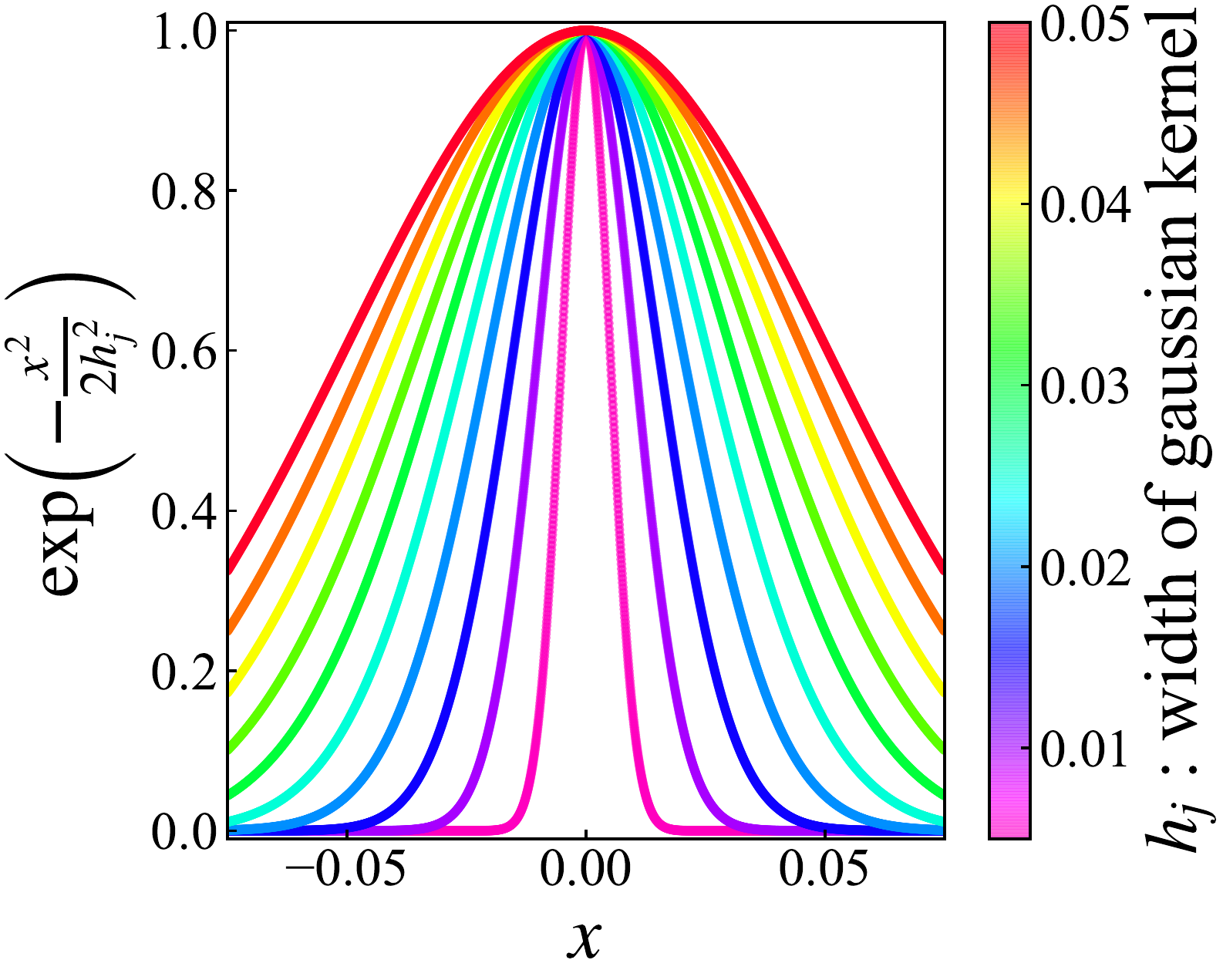}
  \caption{\label{Figmultiplekerenl}  $J=10$ Gaussian kernel functions are plotted. 
  The widths of them are from $h_1=0.005$ to $h_{10}=0.05$ at $0.005$ intervals.}
\end{center}
\end{figure}

\section{Extended Predictive Information Criterion and Bias Correction\label{sec:4}}
\subsection{Extended Predictive Information Criterion\label{sec:4-1}}
The selection of scale parameter $b$ is necessary when we use inverse gamma hyperprior Eq.\ \eqref{alphaInGam}. 
This subsection provides an information criterion to select it. 
Considering that regression is performed using the predictive distribution Eq.\ \eqref{VByosoku2}, 
it is reasonable to use the predictive information criterion (PIC)\cite{genshiro1997information} 
given as 
\begin{equation}
  \mathrm{PIC}=-2\ln h(\bm{y}|\bm{y})+2Bias. \label{PIC}
\end{equation}
Henceforth, we call $\ln h(\bm{y}|\bm{y})$ the log-likelihood for simplicity, following the terminology in Ref. \cite{konishi2008information}. 
The bias correction of log-likelihood is 
\begin{equation}
  Bias=\mathbb{E}_{g(\bm{y})}{\left[\ln{h(\bm{y}|\bm{y})}-\mathbb{E}_{g(\bm{z})}{[\ln{h(\bm{z}|\bm{y})}]}\right]}, \label{Bias}
\end{equation}
where $g(\cdot)$ is true distribution and $\bm{z}$ is future data generated independently on observed $\bm{y}$. 
Eq.\ \eqref{PIC} is derived by minimizing the KL distance between $g(\bm{z})$ and $h(\bm{z}|\bm{y})$ \cite{genshiro1997information,konishi2008information} 
like AIC\cite{akaike1974new}, TIC\cite{takeuchi1976distribution,stone1977asymptotic}, and GIC\cite{konishi1996generalised,konishi2003asymptotic}. 
As shown in later Sec.\ \ref{sec:5}, the MK-VRVM with $\mathrm{InGam}(\bm{\alpha}|a\sim0, b)$ tends to cause over-fitting when $b$ is selected by PIC. 
This is because the number of parameters included in the model, $P=1+NJ$, is larger than data size $N$, 
which is considered the same reason that the model selected by BIC tends to cause over-fitting in $P\gg N$ situation \cite{chen2008extended}. 
The $\mathrm{EBIC}_\gamma$ \cite{chen2008extended}, which is derived without uniform prior over model, solves this problem in BIC. 

As a counterpart of $\mathrm{EBIC}_\gamma$ in PIC, we propose 
\begin{equation}
  \mathrm{EPIC}_{\gamma}=-2\ln{h(\bm{y}|\bm{y})}+2Bias+2\gamma \ln \left({}_{P}\mathrm{C}_{df}\right),\label{EPIC}
\end{equation}
where $\gamma \in [0,1]$ specify prior probability over models and ${}_P\mathrm{C}_{df}$ 
is the binomial coefficients representing the number of models with $df$ as degrees of freedom. 
To derive Eq.\ \eqref{EPIC}, 
we confirm that the PIC can also be obtained from the maximizing the posterior probability of model 
, which is similar to the derivation of the BIC. 
In this process, Eq.\ \eqref{EPIC} is obtained. 
The detail of this work is allocated in Appendix \ref{derivationEPIC}. 
We call Eq. \eqref{EPIC} extended predictive information criterion (EPIC). 

There is no report on the rational determination of $\gamma \in[0,1]$ in $\mathrm{EBIC}_\gamma$. 
The previous studies \cite{chen2008extended,foygel2010extended,chen2012extended} used it, 
specifying several values of $\gamma$. We follow their way in later numerical evaluation. 
11 values of $\gamma$ at intervals of 0.1 from zero to one, i.e. $\gamma=0,0.1,0.2,\cdots,1$, are specified and 
we select $b$ with 11 patterns of $\mathrm{EPIC}_\gamma$. 

Two candidates representing the degrees of freedom $df$ in Eq.\ \eqref{EPIC} are presented. 
The first is the number of relevance vectors, i.e., the weights estimated to be non-zero. 
We denote this number as $RVs$. 
Although $w_{0,1}$ in Eq.\ \eqref{regmulti} does not correspond to input data, 
we include this component into $RVs$ for simplicity if it is estimated to be non-zero. 
The second is $\mathrm{Tr}\bm{H}$ representing the effective degrees of freedom \cite{wahba1990spline,moody1992effective,ye1998measuring} 
where $\bm{H}$ is Eq. \eqref{hatmatrix}. 
In the latter case, the third term in Eq.\ \eqref{EPIC} 
is calculated using binomial coefficients extended to real numbers \cite{sofo2009sums} 
because $\mathrm{Tr}\bm{H}$ might be real number. 
We note that $\beta$ is not included into $df$ in our study 
because it does not contributes to the sparsity unlike $\bm{w}$. 
To obtain reliable calculation results, we performed numerical experiments 
for both cases when $\beta$ was included in $df$ and when it was not. 
It was confirmed that the results are almost the same. Therefore, this issue does not seem to matter. 
We will not show the result including $\beta$ in $df$ for simplicity. 
\subsection{Bias Correction of Log-Likelihood\label{sec:4-2}}
This subsection presents the bias correction term Eq.\ \eqref{Bias} in detail. 
Here, we propose three types of bias correction of log-likelihood. 
The first is analytically calculated. 
Under numerical experiments using artificial data, 
the true distribution $g(\bm{z})$ is clear. 
In our study, artificial data are generated according to $g(\bm{z})=\mathcal{N}(\bm{\mu},\sigma^2\bm{I})$, 
then we could calculate Eq.\ \eqref{Bias} as
\begin{equation}
  Bias_{\mathrm{true}}=\sigma^2\mathrm{Tr}\bm{\Sigma}_\ast^{-1}\bm{H}, \label{Biastrue}
\end{equation}
where $\bm{\Sigma}_\ast$ is Eq.\ \eqref{VByosokuvar} and $\bm{H}$ is Eq.\ \eqref{hatmatrix}. 
Eq.\ \eqref{Biastrue} cannot be applied to the analysis of real data 
because the true distribution is unknown. 
Computable bias correction of log-likelihood in such situations are proposed as the second and third. 
The second is calculated by replacing the true distribution with the plug-in distribution of the variational posterior mean, i.e. 
$g(\bm{z})=\mathcal{N}\left(\bm{\Phi}\mathbb{E}{[\bm{w}]},\mathbb{E}{[\beta]}^{-1}\bm{I}\right)$, then we obtain 
\begin{equation}
  Bias_{\mathrm{plug}}=\mathbb{E}{[\beta]}^{-1}\mathrm{Tr}\bm{\Sigma}_\ast^{-1}\bm{H}. \label{Biasplug}
\end{equation}
The third way is to calculate it in the framework of deriving the GIC \cite{konishi1996generalised,konishi2003asymptotic}. 
According to the Ref. \cite{konishi2008information}, variational predictive distribution Eq.\ \eqref{VByosoku} can be 
approximated as follows: 
\begin{equation}
  h_{\mathrm{VB}}(\bm{z}|\bm{y})=p(\bm{z}|\bm{\hat{w}},\mathbb{E}[\beta])\left\{1+O_p(n^{-1})\right\}, \label{BiasGIC}
\end{equation}
where $p(\cdot)$ obey Eq.\ \eqref{likelihood} and $\hat{\bm{w}}$ is mode of Eq.\ \eqref{qw} which is identical to the mean $\mathbb{E}{[\bm{w}]}$. 
In this case, a statistical functional $\bm{\psi}$ associated with $\hat{\bm{w}}=\mathbb{E}{[\bm{w}]}$ is 
\begin{equation}
  \bm{\psi}(y_n,\bm{w})=\frac{\partial}{\partial\bm{w}}\mathbb{E}{[\ln p(y_n,|\bm{w},\mathbb{E}{[\beta]})]}, \qquad n=1,\cdots,N. \label{psi}
\end{equation} 
Using Eq.\ \eqref{psi}, the bias correction of log-likelihood can be calculated as 
\begin{equation}
  Bias_{\mathrm{GIC}}=\mathrm{Tr}\bm{R}(\bm{\psi},\hat{G})^{-1}\bm{Q}(\bm{\psi},\hat{G}), \label{BiasGIC}
\end{equation}
where $\bm{R}(\bm{\psi},\hat{G})$ and $\bm{Q}(\bm{\psi},\hat{G})$ are given as 
\begin{equation}
  \bm{R}(\bm{\psi},\hat{G})=-\frac{1}{N}\sum_{n=1}^{N}\left.\frac{\partial \bm{\psi}(y_n,\bm{w})}{\partial \bm{w}}\right|_{\bm{w}=\hat{\bm{w}}}=\frac{1}{N}\left\{\mathbb{E}{[\beta]}\bm{\Phi}^{\mathrm{T}}\bm{\Phi}+N\mathbb{E}{[\bm{A}]} \right\}, \label{Rpsi}
\end{equation}
\begin{equation}
  \bm{Q}(\bm{\psi},\hat{G})=\frac{1}{N}\sum_{n=1}^{N}\left.\bm{\psi}(y_n,\bm{w})\frac{\log{p(y_n|\bm{w},\mathbb{E}[\beta])}}{\partial \bm{w}^{\mathrm{T}}}\right|_{\bm{w}=\hat{\bm{w}}}=\frac{1}{N}\left\{\mathbb{E}{[\beta]}^2\bm{\Phi}^{\mathrm{T}}\bm{\Lambda}^2\bm{\Phi}-\mathbb{E}{[\beta]}\mathbb{E}{[\bm{A}]}\bm{\hat{w}}\bm{1}_N^{\mathrm{T}}\bm{\Lambda}\bm{\Phi}\right\}. \label{Qpsi}
\end{equation}
Here $\bm{\Lambda}=\bm{y}-\bm{\Phi}\mathbb{E}{[\bm{w}]}$ and $\bm{1}_N=(1,\cdots,1)^{\mathrm{T}}$ is $N$ dimensional vector whose all components are one. 

We note that the derivative with respect to $\beta$ is not included into Eqs.\ \eqref{Rpsi} and \eqref{Qpsi} in our study. 
The reason is the same as at the end of Sec. \ref{sec:4-1}, i.e., $\beta$ is not related to sparsity unlike $\bm{w}$. 
Even here we also performed numerical experiments for both cases when the derivative with respect to $\beta$ 
was included and when it was not. 
The results are confirmed almost the same. Therefore this issue does not seem to matter. 
We will not show the result including the derivative with respect to $\beta$ for simplicity. 
\section{Numerical Evaluation\label{sec:5}}
This section presents numerical evaluations. 
We generated samples $\left\{(x_n,y_n);n=1,\cdots,N \right\}$ from $y_n=g(x_n)+\varepsilon_n$ with a regression function 
$g(x)$ and error $\varepsilon_n$. 
The input points $\left\{ x_n;n=1,\cdots,N \right\}$ are generated according to the uniform distribution over the interval $[0,1]$. 
Let us assume that 
the error $\varepsilon _n$ are independently distributed according to $\mathcal{N}(0,\sigma^2)$. 
The sample size and standard deviation were $N=50,100$ and $\sigma=0.1,0.3$, respectively. 
The BUMPS, DOPPLER, BLOCKS, and HEAVISINE data \cite{donoho1995adapting} 
were used for $g(x)$. The detail functions of these data are described in Appendix \ref{gx}. 
Henceforth, we will show the results of BUMPS and DOPPLER in detail. 
Those of BLOCKS and HEAVISINE are allocated to an online supplemental material. 

The MK-VRVM regression with $\mathrm{InGam}(\bm{\alpha}|a=10^{-6},b)$ was performed for $100$ Monte Carlo trials, 
by specifying several values of scale parameter $b$. 
For initialization, we set all components of $\mathbb{E}{[\bm{w}]}$ as $0.01$, then Eqs.\ \eqref{pmtil}-\eqref{InGambmtil} and \eqref{varw}-\eqref{dtil} were calculated alternately. 
Due to the restrictions on the calculation resources, the following conditions were used. 
$1)$ The range of scale parameter $b$ is from 0.01 to 10 at intervals of 0.01 and from 10 to 15 at intervals of 1 for $\sigma=0.3$. 
In the case of $\sigma=0.1$, the range is from 0.01 to 10 at intervals of 0.01 and from 10 to 65 at intervals of 1. 
$2)$ When $\alpha_m$ reaches $10^{4}$, that component is not updated any more. 
$3)$ The calculation stops when variational lower bound between two consecutive iteration is smaller than $0.4$. 
$4)$ We regard $\mathbb{E}{[w_m]}$ as relevance vector when absolute value of it is more than 0.03. 
Conversely, weights whose absolute values of expectation are or less $0.03$ are regarded as irrelevance vectors. 
\subsection{Control of sparsity by Inverse Gamma Hyperprior\label{sec:5-1}}
For the aforementioned numerical calculation, 
we evaluate the number of relevance vectors $RVs$ and effective degrees of freedom $\mathrm{Tr}\bm{H}$ 
against scale parameter $b$. 
The results of $df$ against $b$ are shown in Fig. \ref{df_b_BUMPS} for the BUMPS data. 
Furthermore, the regression, corresponding $\mathbb{E}{[\bm{w}]}$ and $\mathbb{E}{[\bm{\alpha}]}$ 
are shown in Fig. \ref{regwalpha_b_BUMPS} in the case of $b=0.01$ and $15.0$. 
Those for DOPPLER data are show in Figs. \ref{df_b_DOPPLER} and \ref{regwalpha_b_DOPPLER}. 
\begin{figure}[H]
  \begin{tabular}{cc}
    \begin{minipage}[t]{0.5\hsize}
      \begin{center}
      \includegraphics[clip,scale=0.5]{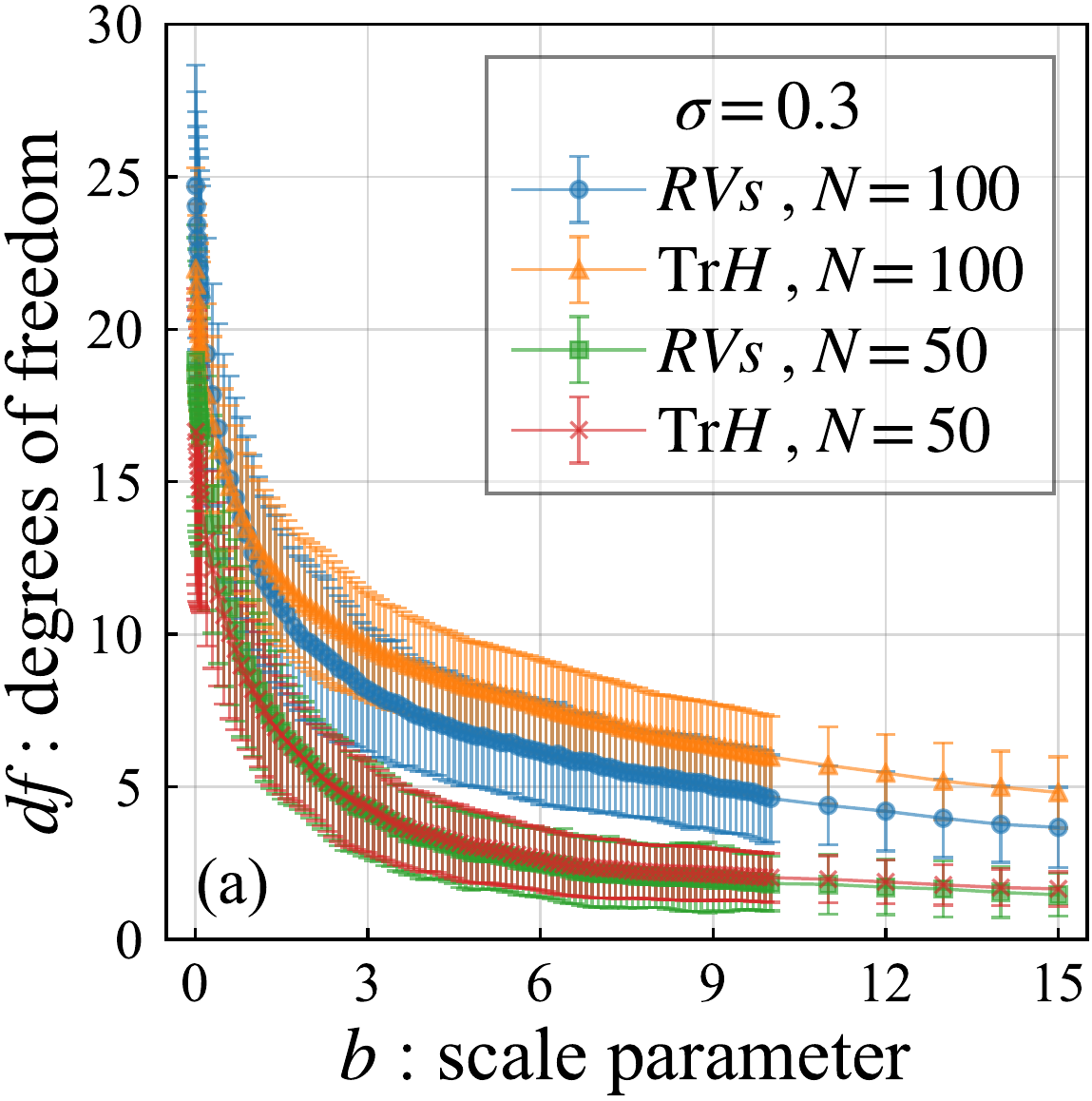}
      \end{center}
    \end{minipage} &
      \begin{minipage}[t]{0.5\hsize}
        \begin{center}
      \includegraphics[clip,scale=0.5]{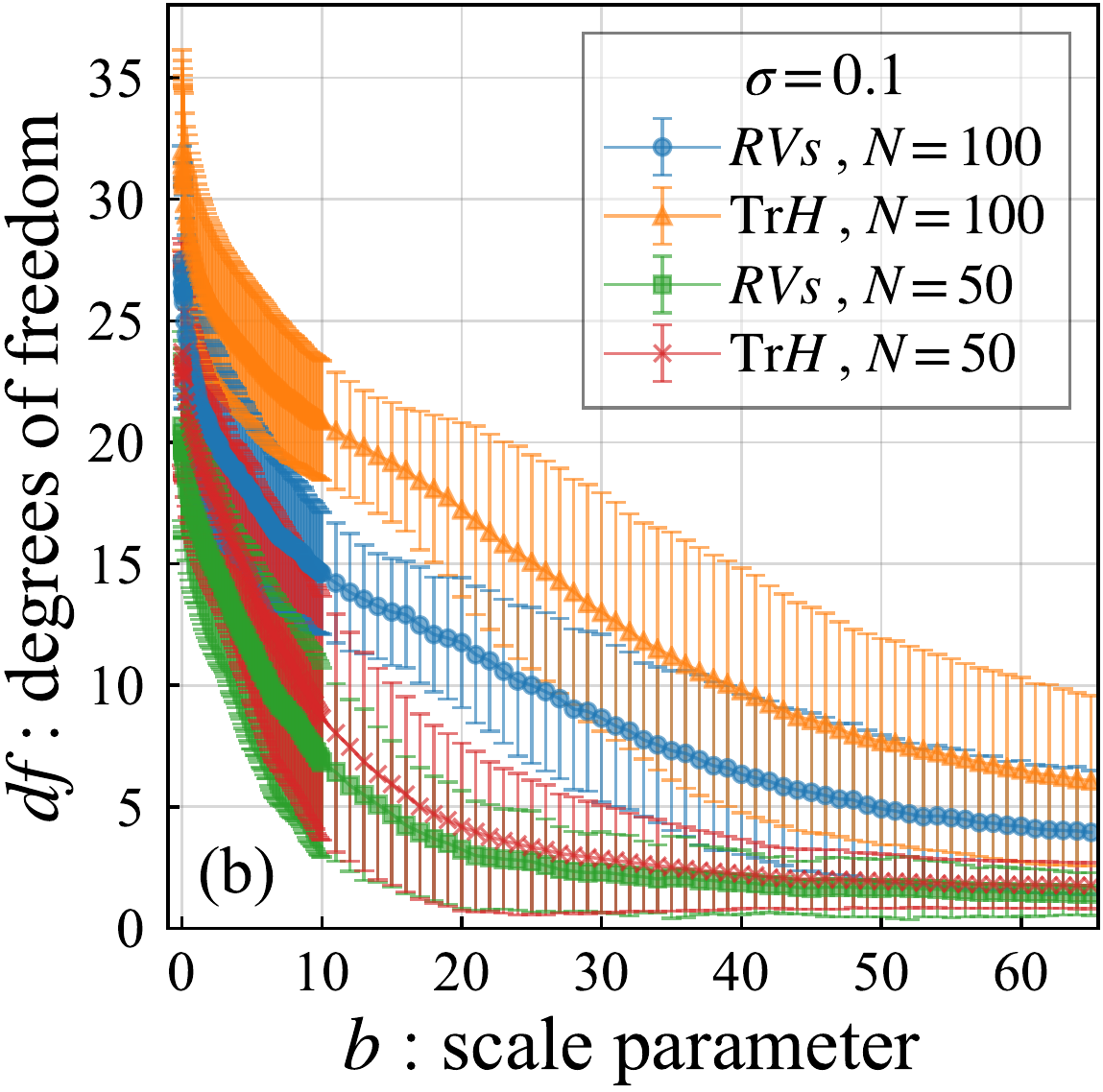}
    \end{center}
    \end{minipage}\\
  \end{tabular}
  \caption{\label{df_b_BUMPS} 
  The degrees of freedom $df$ against scale parameter $b$ in regression for BUMPS data using MK-VRVM with $\mathrm{InGam}(\bm{\alpha}|a=10^{-6},b)$. 
  (a) those in $\sigma=0.3$ case. (b) those in $\sigma=0.1$ case.
  }
\end{figure}
\begin{figure}[H]
  \begin{center}
  \includegraphics[clip,scale=0.5]{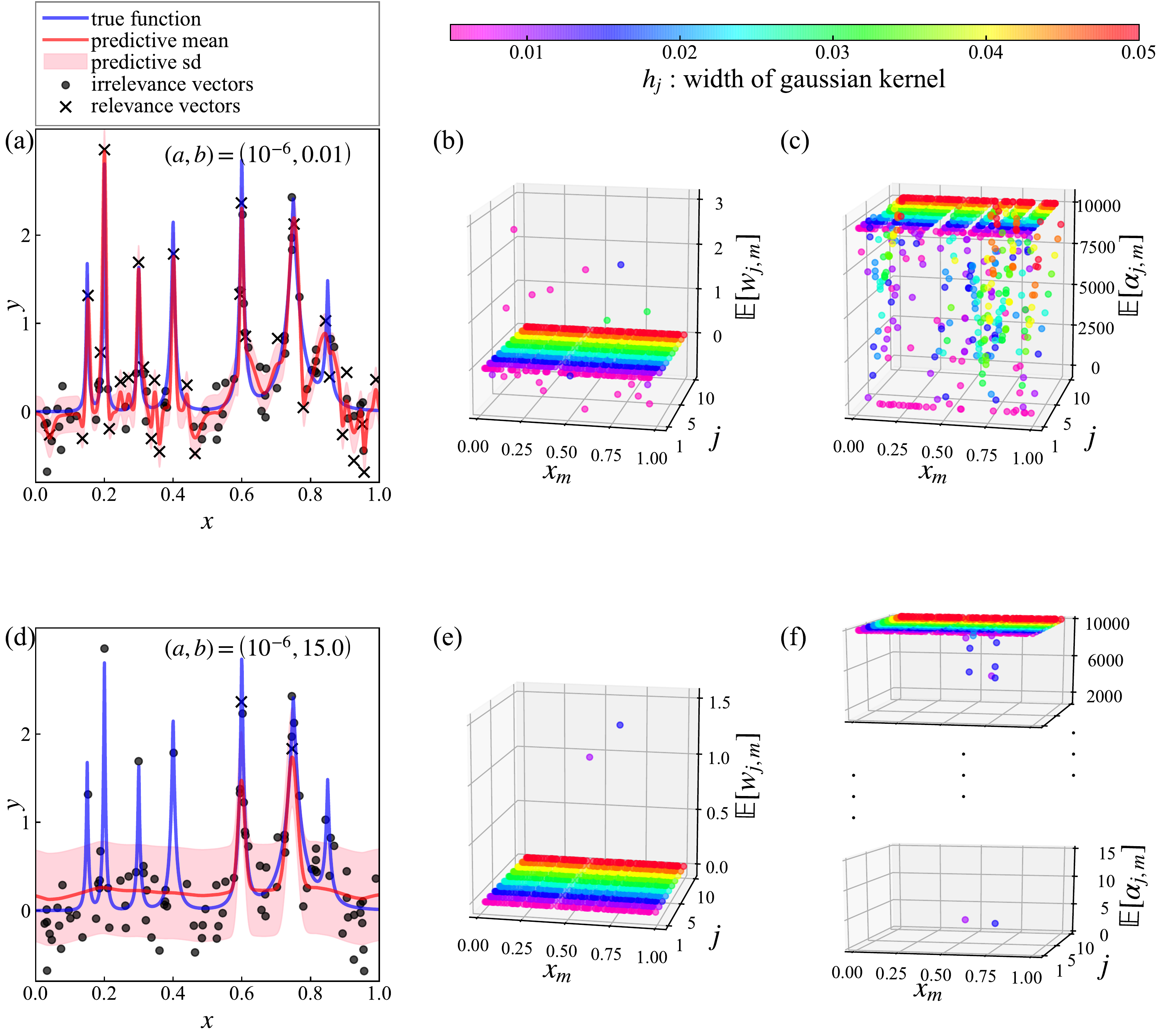}
  \caption{\label{regwalpha_b_BUMPS} 
  (a) An one example of estimated regression model in $100$ Monte Carlo trials for BUMPS data with $(N,\sigma)=(100,0.3)$ using MK-VRVM with $\mathrm{InGam}(\bm{\alpha}|a=10^{-6},b=0.01)$. 
  (b) and (c) corresponding $\mathbb{E}{[\bm{w}]}$ and $\mathbb{E}{[\bm{\alpha}]}$. 
  (d), (e), and (f) Those at $b=15.0$.}
    \end{center}
    \end{figure}
    \begin{figure}[H]
  \begin{tabular}{cc}
    \begin{minipage}[t]{0.5\hsize}
      \begin{center}
      \includegraphics[clip,scale=0.5]{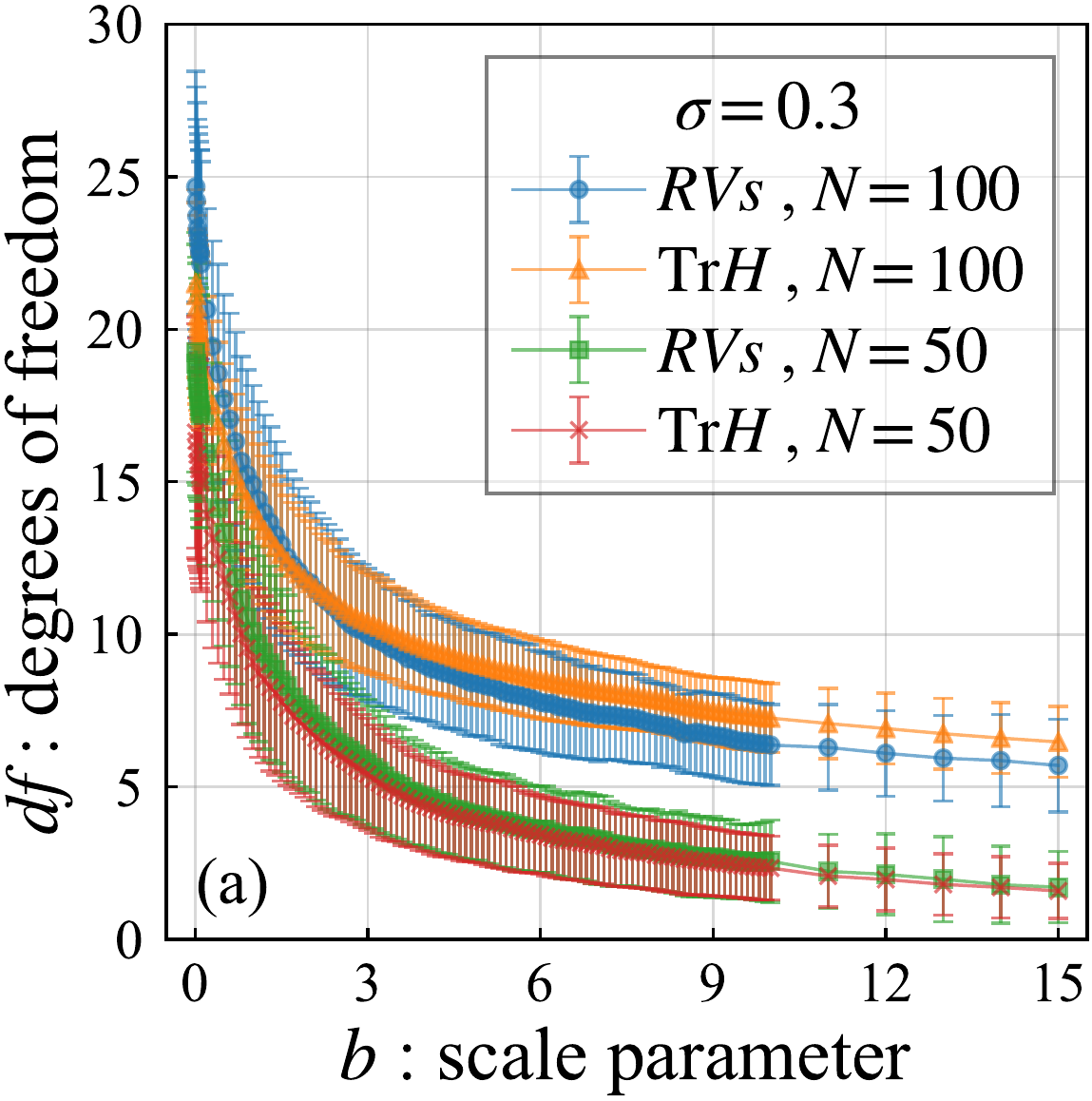}
      \end{center}
    \end{minipage} &
    \begin{minipage}[t]{0.5\hsize}
      \begin{center}
      \includegraphics[clip,scale=0.5]{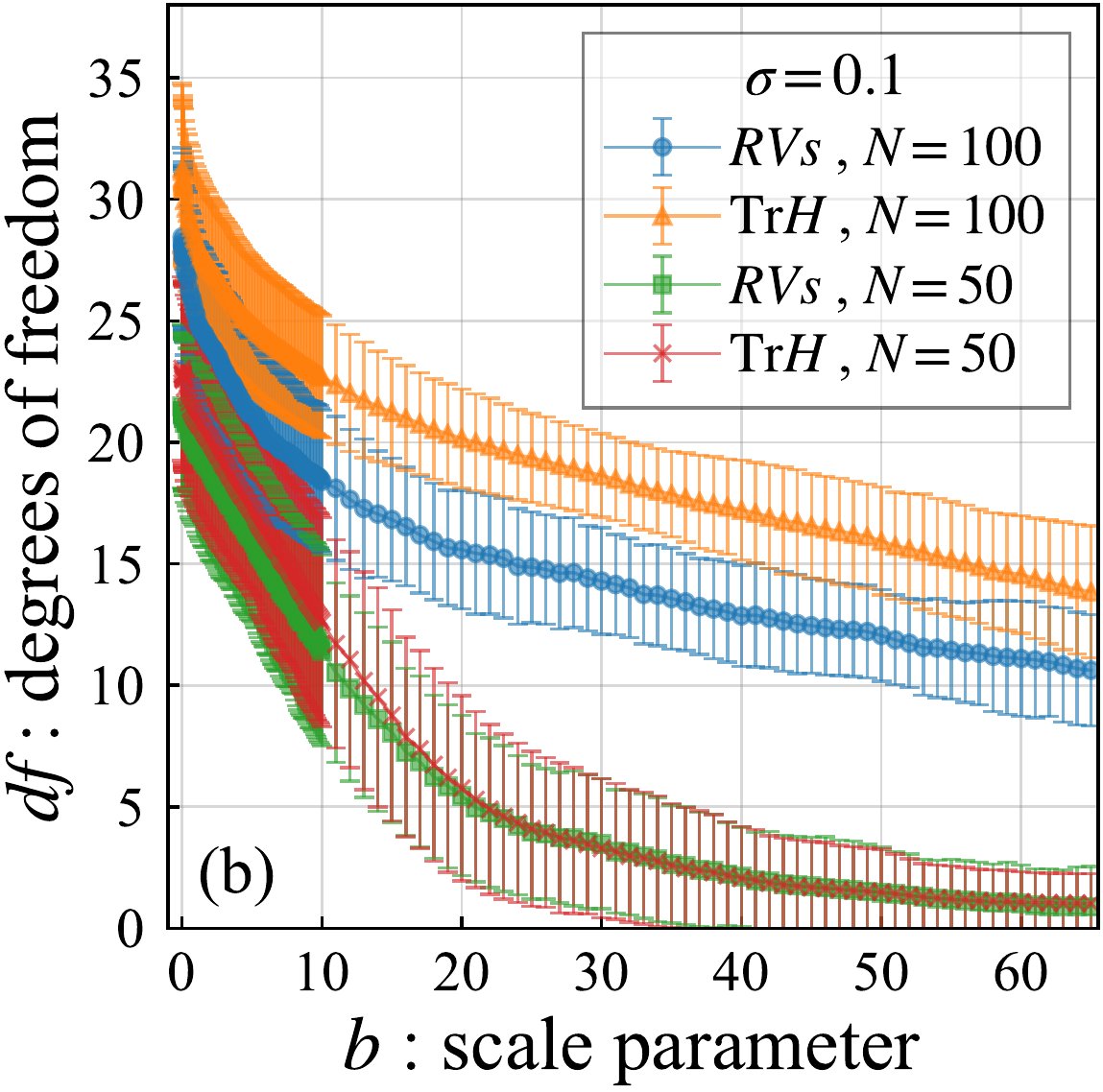}
      \end{center}
    \end{minipage}\\
  \end{tabular}
  \caption{\label{df_b_DOPPLER} The degrees of freedom $df$ against scale parameter $b$ in regression for DOPPLER data using MK-VRVM with $\mathrm{InGam}(\bm{\alpha}|a=10^{-6},b)$. 
  (a) those in $\sigma=0.3$ case. (b) those in $\sigma=0.1$ case.}
\end{figure}

\begin{figure}[H]
  \begin{center}
  \includegraphics[clip,scale=0.5]{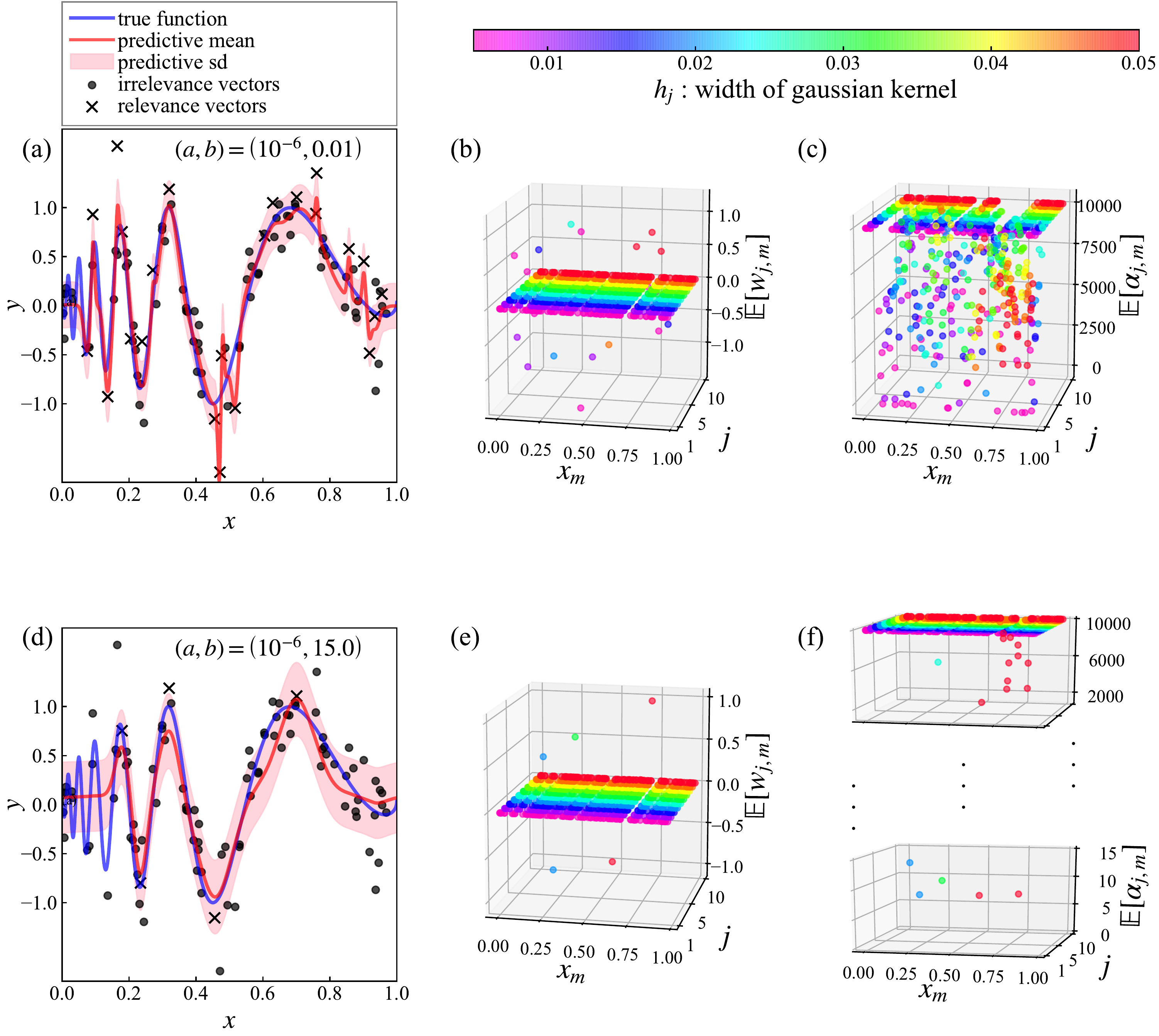}
  \caption{\label{regwalpha_b_DOPPLER} 
  (a) An one example of estimated regression model in $100$ Monte Carlo trials for DOPPLER data with $(N,\sigma)=(100,0.3)$ using MK-VRVM with $\mathrm{InGam}(\bm{\alpha}|a=10^{-6},b=0.01)$. 
  (b) and (c) corresponding $\mathbb{E}{[\bm{w}]}$ and $\mathbb{E}{[\bm{\alpha}]}$. 
  (d), (e), and (f) Those at $b=15.0$.}
    \end{center}
    \end{figure}
As shown in Fig. \ref{df_b_BUMPS}(a), the degrees of freedom increase 
and estimated model causes over-fitting (see Fig. \ref{regwalpha_b_BUMPS}(a)(b)(c)) 
when scale parameter $b$ is close to zero. 
We recall that $\mathrm{InGam}(\bm{\alpha}|a\sim 0,b\sim 0)$ 
correspond to non-informative hyperprior. 
Conversely, $df$ decrease and estimated model causes under-fitting 
(see Fig. \ref{regwalpha_b_BUMPS}(d)) when scale parameter $b$ is away from zero. 
The hyperprior $\mathrm{InGam}(\bm{\alpha}|a\sim 0,b\neq 0)$ 
contains the information that makes the model sparse. 
This property is visually confirmed in Figs. \ref{regwalpha_b_BUMPS}(e)(f). 
These results are agreement with $p(w_m)$ and $p(\alpha_m)$ shown in Figs. \ref{comp_priors}(a)(b). 
The Fig. \ref{df_b_BUMPS}(b) in the case of $\sigma=0.1$ also supports this property. 
According to Fig. \ref{df_b_BUMPS}(b), the large $b$ 
seems to be required for decreasing $df$ compared with Fig. \ref{df_b_BUMPS}(a) in the case of $\sigma=0.3$. 
The reason is that the true structure in $g(x)$ appears clearly in $\sigma=0.1$ 
and the large $df$ needed to follow the non-homogeneous structure. 

The similar interpretations above are also applied to Figs. \ref{df_b_DOPPLER} and \ref{regwalpha_b_DOPPLER}. 
The degrees of freedom against $b$ seem to differ slightly depending on the types of data. 
In other words, $df$ for DOPPLER is less likely to decrease than that for BUMPS. 
To see this property, for example, Fig. \ref{df_b_BUMPS}(b) is compared to Fig. \ref{df_b_DOPPLER}(b) in case of $N=100$. 
Figs. \ref{regwalpha_b_BUMPS}(d)(e)(f) and Figs. \ref{regwalpha_b_DOPPLER}(d)(e)(f) may also mean this property. 

It is not clear whether it makes sense to compare the values of $RVs$ with those of $\mathrm{Tr}\bm{H}$. 
The reason is that $RVs$ depends on the threshold of relevance vectors, 
and in our calculation, we do not use a strict threshold due to restriction on computational resource. 
Another reason is that $\bm{H}$ is calculated by applying the approximation to the integral with respect to $\beta$, 
as explained in Sec. \ref{sec:2-2}.
\subsection{Predictive Accuracy\label{sec:5-2}}
For the aforementioned calculations, we evaluate predictive accuracy when scale parameter $b$ is selected by 
$\left\{\mathrm{EPIC}_\gamma;\gamma=0,0.1,0.2,\cdots,1\right\}$. 
MK-VRVM and SK-VRVM with $\mathrm{Gam}(\bm{\alpha}|a=10^{-6},b=10^{-6})$ were used as comparison methods. 
The calculation stops when variational lower bounds between two consecutive iteration is smaller than $0.01$ for the former and $0.00001$ for the latter. 
We note that these comparison methods do not converge under the convergence criterion of MK-VRVM with $\mathrm{InGam}(\bm{\alpha}|a = 10^{-6},b)$. 
The MK-RVM and SK-RVM, which are estimated by second type maximum likelihood 
were also prepared for the comparison (see Refs. \cite{tipping2000relevance,tipping2001sparse} for the algorithm). 
The calculation stops when all of the updating parameters between two consecutive iteration are smaller than $0.01$ for the former and $0.005$ for the latter. 
The initialization, the upper limit of $\mathbb{E}{[\alpha_m]}$, and threshold determining the relevance vectors 
were the same as those of the proposed method. 
The widths of Gaussian kernels in the SK-VRVM and SK-RVM were set to $h=0.005, 0.05, 0.0275$, 
which represent the narrowest width, the widest width, and the middle width among the 10 Gaussian kernels used in the MK-VRVM and MK-RVM. 
Comparison methods for $\mathrm{EPIC}_\gamma$ were CV\cite{stone1974cross}, GCV\cite{craven1978smoothing,golub1979generalized}, and PIC\cite{genshiro1997information} 
which is identical to $\mathrm{EPIC}_{\gamma=0}$. 
The evaluation points are following:
\begin{itemize}
  \item the mean squared errors (MSE) defined as $\mathrm{MSE}=\sum_{n=1}^{N}\left\{\hat{y}_n-g(x_n)\right\}^2/(N-1)$ where 
  $\hat{\bm{y}}=(\hat{y}_1,\cdots,\hat{y}_N)^{\mathrm{T}}$ is obtained by Eq.\ \eqref{VByosokumean}, i.e., $\hat{\bm{y}}=\mathbb{E}{[\beta]}\bm{\Phi}\tilde{\bm{\Sigma}}\bm{\Phi}^{\mathrm{T}}\bm{y}$. 
  \item the predictive squared error (PSE) defined as $\mathrm{PSE}=\sum_{i=1}^{1000}\left\{\hat{z}_i-g(x^{\mathrm{new}}_i)\right\}^2/(1000-1)$ 
  where new input data $\left\{ x^{\mathrm{new}}_i;i=1,\cdots,1000 \right\}$ are uniformly spaced on $[0,1]$ and 
  $\hat{\bm{z}}=(\hat{z}_1,\cdots,\hat{z}_{1000})^{\mathrm{T}}$ is obtained by $\hat{\bm{z}}=\bm{\Phi}^{\mathrm{new}}\tilde{\bm{\mu}}$. 
  Here $\tilde{\bm{\mu}}$ and $\bm{\Phi}^{\mathrm{new}}$ are Eq.\ \eqref{meanw} and the design matrix of new inputs data, respectively. 
  \item the number of relevance vectors $RVs$. 
  Although not only $RVs$ but also $\mathrm{Tr}\bm{H}$ is used as $df$ of ${}_P\mathrm{C}_{df}$ in $\mathrm{EPIC}_\gamma$, 
  henceforth, the sparsity in estimated model is represented by $RVs$ for simplicity. 
  \item the percentage of sparsity defined as $100\times(RVs/P)$ where $P=1+NJ$ for the multiple kernel case and $P=1+N$ for the single kernel case. 
  \item the values of scale parameter $b$ selected by model selection criteria. 
\end{itemize}

We should mention that the proposed method does not always show good performance for all data types, $N$, and $\sigma$. 
The methods with the smallest PSE with respect to data types, $N$, and $\sigma$ 
are briefly shown in Table \ref{kani_tab_BUMPS}-\ref{kani_tab_HEAVISINE}. 
  \begin{table}[H]
    \caption{\label{kani_tab_BUMPS}Good performance methods for regression to BUMPS data}
\begin{center}
\begin{tabular}{c|c|c} 
 & $\sigma=0.1$ & $\sigma=0.3$ \\\hline
$N=50$ & MK-VRVM with $\mathrm{InGam}(\bm{\alpha}|a = 10^{-6},b)$ & MK-VRVM with $\mathrm{InGam}(\bm{\alpha}|a = 10^{-6},b)$ \\\hline
$N=100$ & MK-RVM & MK-VRVM with $\mathrm{InGam}(\bm{\alpha}|a = 10^{-6},b)$ \\
\end{tabular}
\end{center}
\end{table}
  \begin{table}[H]
        \caption{\label{kani_tab_DOPPLER}Good performance methods for regression to DOPPLER data}
    \begin{center}
    \begin{tabular}{c|c|c} 
     & $\sigma=0.1$ & $\sigma=0.3$ \\\hline
    $N=50$ & MK-VRVM with $\mathrm{InGam}(\bm{\alpha}|a = 10^{-6},b)$ & MK-VRVM with $\mathrm{InGam}(\bm{\alpha}|a = 10^{-6},b)$ \\\hline
    $N=100$ & \begin{tabular}{c}MK-VRVM with $\mathrm{InGam}(\bm{\alpha}|a = 10^{-6},b)$ \\ and MK-RVM \end{tabular} & MK-VRVM with $\mathrm{InGam}(\bm{\alpha}|a = 10^{-6},b)$ \\
    \end{tabular}
    \end{center}
    \end{table}
   \begin{table}[H]
    \caption{\label{kani_tab_BLOCKS}Good performance methods for regression to BLOCKS data}
\begin{center}
\begin{tabular}{c|c|c} 
 & $\sigma=0.1$ & $\sigma=0.3$ \\\hline
$N=50$ & SK-RVM with $h=0.0275$ & SK-RVM with $h=0.0275$ \\\hline
$N=100$ & MK-VRVM with $\mathrm{InGam}(\bm{\alpha}|a = 10^{-6},b)$ &\begin{tabular}{c}MK-VRVM with $\mathrm{InGam}(\bm{\alpha}|a = 10^{-6},b)$\\and SK-RVM with $h=0.0275$\end{tabular} \\
\end{tabular}
\end{center}
\end{table}
    \begin{table}[H]
      \caption{\label{kani_tab_HEAVISINE}Good performance methods for regression to HEAVISINE data}
  \begin{center}
  \begin{tabular}{c|c|c} 
   & $\sigma=0.1$ & $\sigma=0.3$ \\\hline
  $N=50$ & SK-RVM with $h=0.0275$ & SK-RVM with $h=0.0275$ \\\hline
  $N=100$ & MK-VRVM with $\mathrm{InGam}(\bm{\alpha}|a = 10^{-6},b)$ &\begin{tabular}{c}SK-VRVM with gamma hyperprior \\and SK-RVM with $h=0.05$\end{tabular} \\
  \end{tabular}
  \end{center}
  \end{table}
According to Table \ref{kani_tab_BUMPS}-\ref{kani_tab_HEAVISINE}, 
our proposed method seems to show good performance for BUMPS and DOPPLER data. 
Henceforth, we show the concrete results for the BUMPS and DOPPLER data with $N=100,\sigma=0.3$ in detail. 
The numerical results are summarized in Table \ref{tab_BUMPSN100std03} for the BUMPS. 
The results of $\mathrm{EPIC}_\gamma$ that minimize PSE among $\left\{\mathrm{EPIC}_\gamma;\gamma=0,0.1,\cdots,1\right\}$ are listed in the table. 
Furthermore, the regression, corresponding $\mathbb{E}{[\bm{w}]}$ and $\mathbb{E}{[\bm{\alpha}]}$ 
are visualized. Those of MK-VRVM with $\mathrm{InGam}(\bm{\alpha}|a = 10^{-6},b)$ (proposed method) are shown in Fig. \ref{regwalpha_EPIC_BUMPS} 
when scale parameter $b$ is selected by $\mathrm{EPIC}_\gamma$. 
Those of MK-VRVM and SK-VRVM with $\mathrm{Gam}(\bm{\alpha}|a=10^{-6},b=10^{-6})$ (comparison methods) are shown in Fig. \ref{regwalpha_VmVmono_BUMPS}. 
The results for the DOPPLER are shown in Table \ref{tab_DOPPLERN100std03} and Figs. \ref{regwalpha_EPIC_DOPPLER} and \ref{regwalpha_VmVmono_DOPPLER}. 
\begin{table}[H]
  \caption{\label{tab_BUMPSN100std03}Results for regression to BUMPS data $(N = 100$, $\sigma = 0.3)$}
  \begin{center}\scalebox{0.65}{
  \begin{tabular}{ccccccccccccc} 
  method & $Bias$ & $df$ of ${}_P\mathrm{C}_{df}$ & IC & \multicolumn{2}{c}{$\mathrm{MSE}\times 10^{-2}$} & \multicolumn{2}{c}{$\mathrm{PSE}\times 10^{-2}$} & \multicolumn{2}{c}{$RVs$} & \% of sparsity & \multicolumn{2}{c}{$b$} \\\hline
  MK-VRVM with inverse gamma & $Bias_{\mathrm{true}}$ & n/a & $\mathrm{EPIC}_{0}$ & 3.817  & (1.007) & 9.161  & (2.871) & 21.86  & (3.59) & 2.184  & 0.153  & (0.238) \\
   &  & $RVs$ & $\mathrm{EPIC}_{0.4}$ & 2.963  & (0.816) & 8.362  & (2.795) & 8.67  & (2.23) & 0.866  & 2.268  & (1.232) \\
   &  & $\mathrm{Tr}\bm{H}$ & $\mathrm{EPIC}_{0.3}$ & 3.042  & (0.886) & 8.353  & (2.759) & 12.24  & (3.14) & 1.223  & 1.462  & (0.799) \\\cline{2-13}
   & $Bias_{\mathrm{plug}}$ & n/a & $\mathrm{EPIC}_{0}$ & 4.042  & (1.049) & 9.161  & (2.880) & 24.67  & (3.94) & 2.465  & 0.012  & (0.008) \\
   &  & $RVs$ & $\mathrm{EPIC}_{0.5}$ & 3.023  & (0.946) & 8.453  & (2.848) & 8.81  & (2.98) & 0.880  & 2.242  & (1.197) \\
   &  & $\mathrm{Tr}\bm{H}$ & $\mathrm{EPIC}_{0.6}$ & 3.012  & (0.927) & 8.462  & (2.781) & 9.35  & (2.92) & 0.934  & 2.912  & (1.487) \\\cline{2-13}
   & $Bias_{\mathrm{GIC}}$ & n/a & $\mathrm{EPIC}_{0}$ & 4.041  & (1.045) & 9.156  & (2.873) & 24.70  & (3.97) & 2.468  & 0.014  & (0.021) \\
   &  & $RVs$ & $\mathrm{EPIC}_{0.6}$ & 3.065  & (0.929) & 8.500  & (2.834) & 8.62  & (3.01) & 0.861  & 2.371  & (1.370) \\
   &  & $\mathrm{Tr}\bm{H}$ & $\mathrm{EPIC}_{0.7}$ & 3.016  & (0.923) & 8.445  & (2.800) & 9.32  & (2.87) & 0.931  & 3.179  & (1.655) \\\cline{2-13}
   & n/a & n/a & CV & 3.808  & (1.053) & 8.895  & (2.709) & 22.03  & (5.39) & 2.201  & 0.133  & (0.289) \\\cline{2-13}
   & n/a & n/a & GCV & 4.040  & (1.050) & 9.158  & (2.880) & 24.60  & (3.93) & 2.458  & 0.014  & (0.010) \\\hline
  MK-VRVM with gamma & n/a & n/a & n/a & 4.246  & (1.066) & 9.401  & (3.003) & 27.01  & (4.04) & 2.698  & n/a & n/a \\\hline
  SK-VRVM with gamma $h=0.005$ & n/a & n/a & n/a & 4.280  & (0.961) & 11.925  & (3.737) & 35.02  & (4.07) & 34.673  & n/a & n/a \\\hline
  SK-VRVM with gamma $h=0.0275$ & n/a & n/a & n/a & 7.383  & (2.580) & 12.664  & (2.626) & 9.84  & (2.35) & 9.743  & n/a & n/a \\\hline
  SK-VRVM with gamma $h=0.05$ & n/a & n/a & n/a & 12.282  & (3.804) & 16.405  & (1.793) & 7.23  & (2.09) & 7.158  & n/a & n/a \\\hline
  MK-RVM  & n/a & n/a & n/a & 4.173  & (1.059) & 9.329  & (2.952) & 26.86  & (3.92) & 2.683  & n/a & n/a \\\hline
  SK-RVM $h=0.005$ & n/a & n/a & n/a & 4.274  & (0.956) & 11.917  & (3.738) & 33.91  & (4.16) & 33.574  & n/a & n/a \\\hline
  SK-RVM $h=0.0275$ & n/a & n/a & n/a & 7.399  & (2.579) & 12.708  & (2.714) & 9.49  & (1.71) & 9.396  & n/a & n/a \\\hline
  SK-RVM $h=0.05$ & n/a & n/a & n/a & 12.230  & (3.785) & 16.378  & (1.833) & 6.54  & (1.37) & 6.475  & n/a & n/a \\\hline
  \end{tabular}}
  \end{center}
  \end{table}
  \begin{figure}[H]
    \begin{center}
    \includegraphics[clip,scale=0.7]{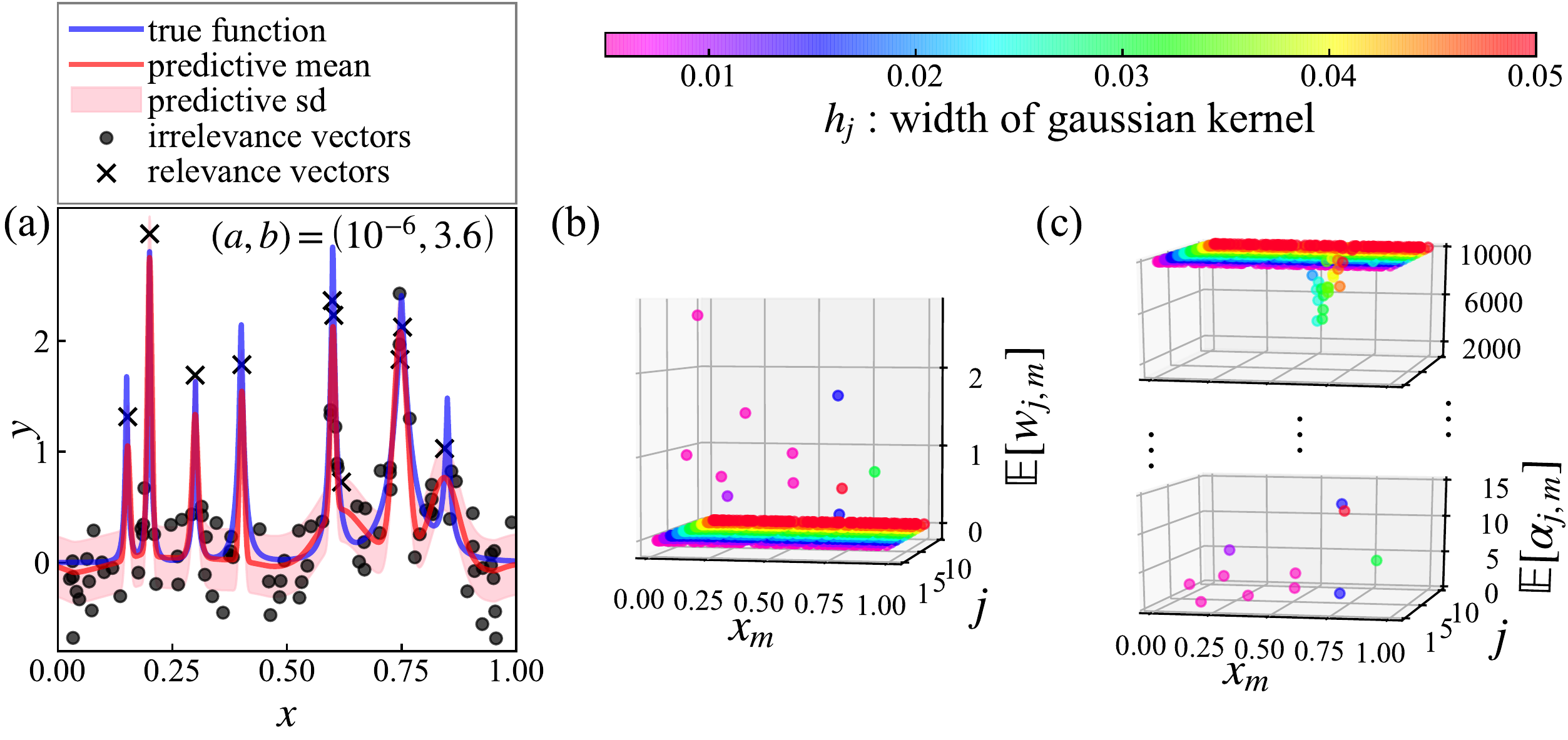}
    \caption{\label{regwalpha_EPIC_BUMPS} 
    (a) An one example of estimated regression model in $100$ Monte Carlo trials for BUMPS data with $(N,\sigma)=(100,0.3)$ 
    using MK-VRVM with $\mathrm{InGam}(\bm{\alpha}|a = 10^{-6},b)$. The scale parameter $b$ is selected by $\mathrm{EPIC}_{0.7}$ with $Bias_{\mathrm{GIC}}$ and ${}_P\mathrm{C}_{df=\mathrm{Tr}\bm{H}}$. 
  (b) and (c) corresponding $\mathbb{E}{[\bm{w}]}$ and $\mathbb{E}{[\bm{\alpha}]}$.
    }
    \end{center}
    \end{figure}
    \begin{figure}[H]
    \begin{center}
      \includegraphics[clip,scale=0.6]{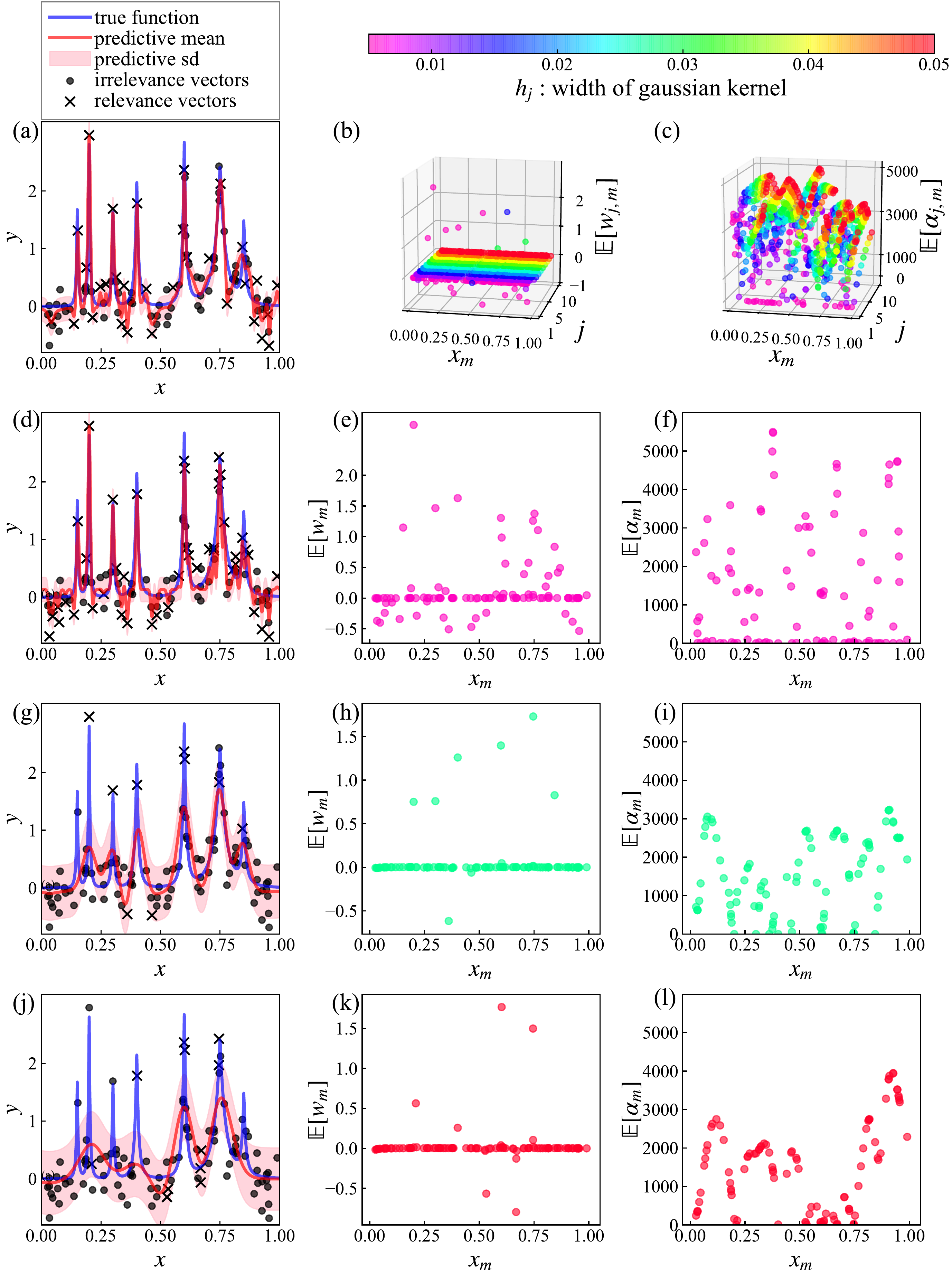}
      \caption{\label{regwalpha_VmVmono_BUMPS} 
      (a) An one example of estimated regression model in $100$ Monte Carlo trials for BUMPS data with $(N,\sigma)=(100,0.3)$ 
      using MK-VRVM with $\mathrm{Gam}(\bm{\alpha}|a=10^{-6},b=10^{-6})$.  
      (b) and (c) corresponding $\mathbb{E}{[\bm{w}]}$ and $\mathbb{E}{[\bm{\alpha}]}$.
      (d), (e), and (f) Those for SK-VRVM with $\mathrm{Gam}(\bm{\alpha}|a=10^{-6},b=10^{-6})$ in the case of $h=0.005$. 
      (g), (h), and (i) Those in the case of $h=0.0275$. 
      (j), (k), and (l) Those in the case of $h=0.05$.}
    \end{center}
        \end{figure}
\begin{table}[H]
  \caption{\label{tab_DOPPLERN100std03}Results for regression to DOPPLER data $(N = 100$, $\sigma = 0.3)$}
  \begin{center}\scalebox{0.68}{
  \begin{tabular}{ccccccccccccc} 
  method & $Bias$ & $df$ of ${}_P\mathrm{C}_{df}$ & IC & \multicolumn{2}{c}{$\mathrm{MSE}\times 10^{-2}$} & \multicolumn{2}{c}{$\mathrm{PSE}\times 10^{-2}$} & \multicolumn{2}{c}{$RVs$} & \% of sparsity & \multicolumn{2}{c}{$b$} \\\hline
  MK-VRVM with inverse gamma & $Bias_{\mathrm{true}}$ & n/a & $\mathrm{EPIC}_{0}$ & 3.849  & (0.910) & 5.325  & (1.403) & 22.32  & (3.18) & 2.230  & 0.213  & (0.475) \\
   &  & $RVs$ & $\mathrm{EPIC}_{0.4}$ & 3.047  & (0.787) & 4.444  & (1.384) & 9.64  & (2.34) & 0.963  & 2.870  & (1.408) \\
   &  & $\mathrm{Tr}\bm{H}$ & $\mathrm{EPIC}_{0.4}$ & 2.998  & (0.724) & 4.352  & (1.355) & 11.70  & (2.17) & 1.169  & 2.424  & (1.205) \\\cline{2-13}
   & $Bias_{\mathrm{plug}}$ & n/a & $\mathrm{EPIC}_{0}$ & 4.066  & (0.922) & 5.325  & (1.400) & 24.66  & (3.77) & 2.464  & 0.011  & (0.005) \\
   &  & $RVs$ & $\mathrm{EPIC}_{0.5}$ & 3.168  & (0.872) & 4.549  & (1.450) & 10.24  & (3.62) & 1.023  & 2.626  & (1.337) \\
   &  & $\mathrm{Tr}\bm{H}$ & $\mathrm{EPIC}_{0.5}$ & 3.145  & (0.841) & 4.472  & (1.410) & 13.13  & (3.96) & 1.312  & 2.052  & (1.417) \\\cline{2-13}
   & $Bias_{\mathrm{GIC}}$ & n/a & $\mathrm{EPIC}_{0}$ & 4.064  & (0.919) & 5.322  & (1.398) & 24.68  & (3.75) & 2.466  & 0.012  & (0.006) \\
   &  & $RVs$ & $\mathrm{EPIC}_{0.6}$ & 3.182  & (0.862) & 4.572  & (1.413) & 9.71  & (3.09) & 0.970  & 3.121  & (1.989) \\
   &  & $\mathrm{Tr}\bm{H}$ & $\mathrm{EPIC}_{0.6}$ & 3.119  & (0.827) & 4.467  & (1.420) & 12.85  & (4.03) & 1.284  & 2.496  & (2.012) \\\cline{2-13}
   & n/a & n/a & CV & 3.784  & (0.930) & 5.028  & (1.447) & 21.79  & (4.77) & 2.177  & 0.175  & (0.270) \\\cline{2-13}
   & n/a & n/a & GCV & 4.065  & (0.922) & 5.330  & (1.409) & 24.67  & (3.76) & 2.465  & 0.012  & (0.007) \\\hline
  MK-VRVM with gamma & n/a & n/a & n/a & 4.256  & (0.953) & 5.665  & (1.606) & 27.18  & (3.80) & 2.715  & n/a & n/a \\\hline
  SK-VRVM with gamma h=0.005 & n/a & n/a & n/a & 5.685  & (0.943) & 13.333  & (1.875) & 45.31  & (4.23) & 44.861  & n/a & n/a \\\hline
  SK-VRVM with gamma h=0.0275 & n/a & n/a & n/a & 3.378  & (0.897) & 4.999  & (1.169) & 14.32  & (2.39) & 14.178  & n/a & n/a \\\hline
  SK-VRVM with gamma h=0.05 & n/a & n/a & n/a & 4.858  & (1.375) & 6.486  & (1.364) & 9.24  & (2.56) & 9.149  & n/a & n/a \\\hline
  MK-RVM  & n/a & n/a & n/a & 4.192  & (0.918) & 5.561  & (1.548) & 26.55  & (3.84) & 2.652  & n/a & n/a \\\hline
  SK-RVM h=0.005 & n/a & n/a & n/a & 5.693  & (0.941) & 13.306  & (1.896) & 43.96  & (4.28) & 43.525  & n/a & n/a \\\hline
  SK-RVM h=0.0275 & n/a & n/a & n/a & 3.367  & (0.892) & 4.975  & (1.162) & 13.44  & (1.75) & 13.307  & n/a & n/a \\\hline
  SK-RVM h=0.05 & n/a & n/a & n/a & 4.844  & (1.361) & 6.456  & (1.363) & 8.17  & (1.48) & 8.089  & n/a & n/a \\\hline
  \end{tabular}}
  \end{center}
  \end{table}
  \begin{figure}[H]
    \begin{center}
\includegraphics[clip,scale=0.7]{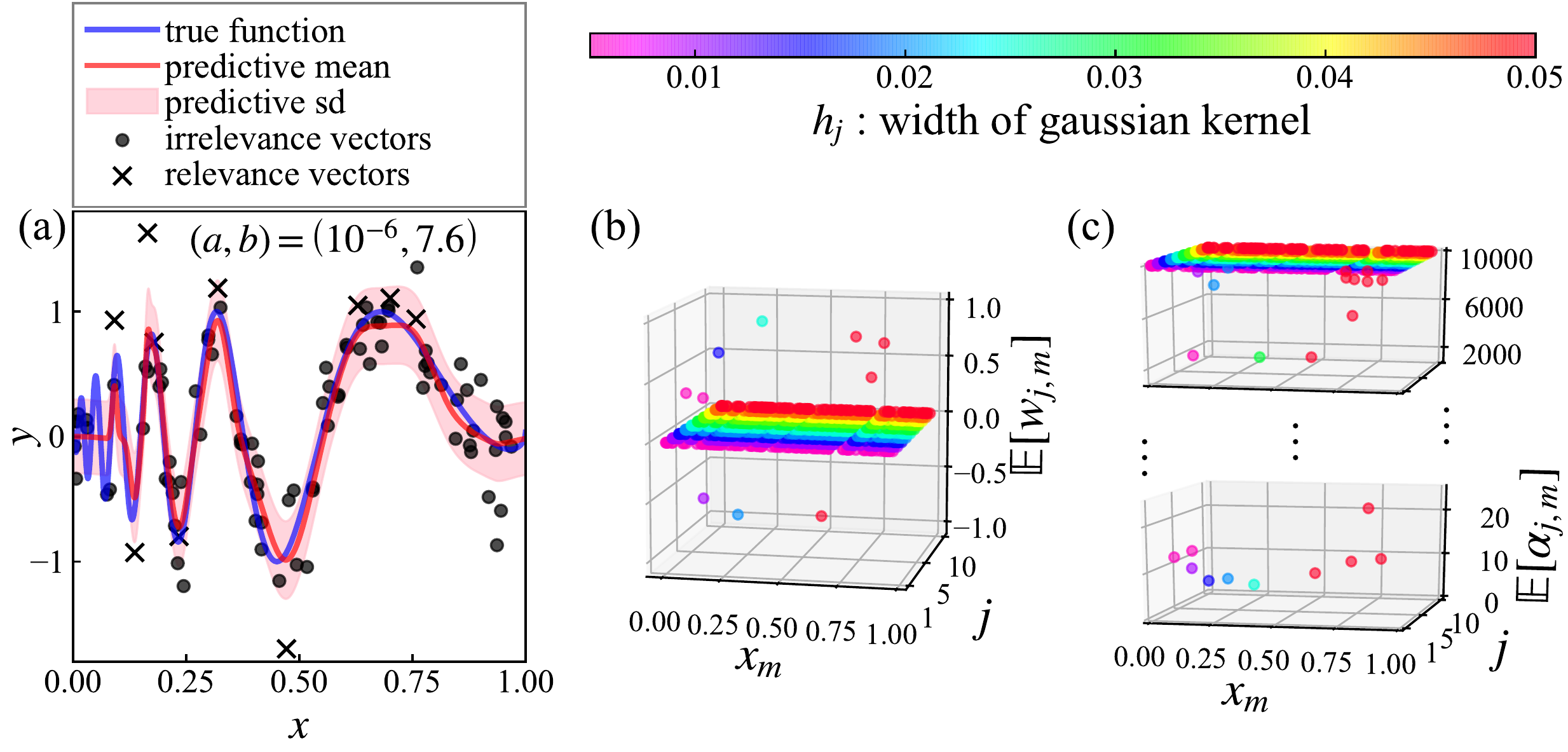}
\caption{\label{regwalpha_EPIC_DOPPLER} 
(a) An one example of estimated regression model in $100$ Monte Carlo trials for DOPPLER data with $(N,\sigma)=(100,0.3)$ 
    using MK-VRVM with $\mathrm{InGam}(\bm{\alpha}|a = 10^{-6},b)$. The scale parameter $b$ is selected by $\mathrm{EPIC}_{0.6}$ with $Bias_{\mathrm{GIC}}$ and ${}_P\mathrm{C}_{df=\mathrm{Tr}\bm{H}}$. 
  (b) and (c) corresponding $\mathbb{E}{[\bm{w}]}$ and $\mathbb{E}{[\bm{\alpha}]}$.
    }
  \end{center}
  \end{figure}
  \begin{figure}[H]
  \begin{center}
    \includegraphics[clip,scale=0.6]{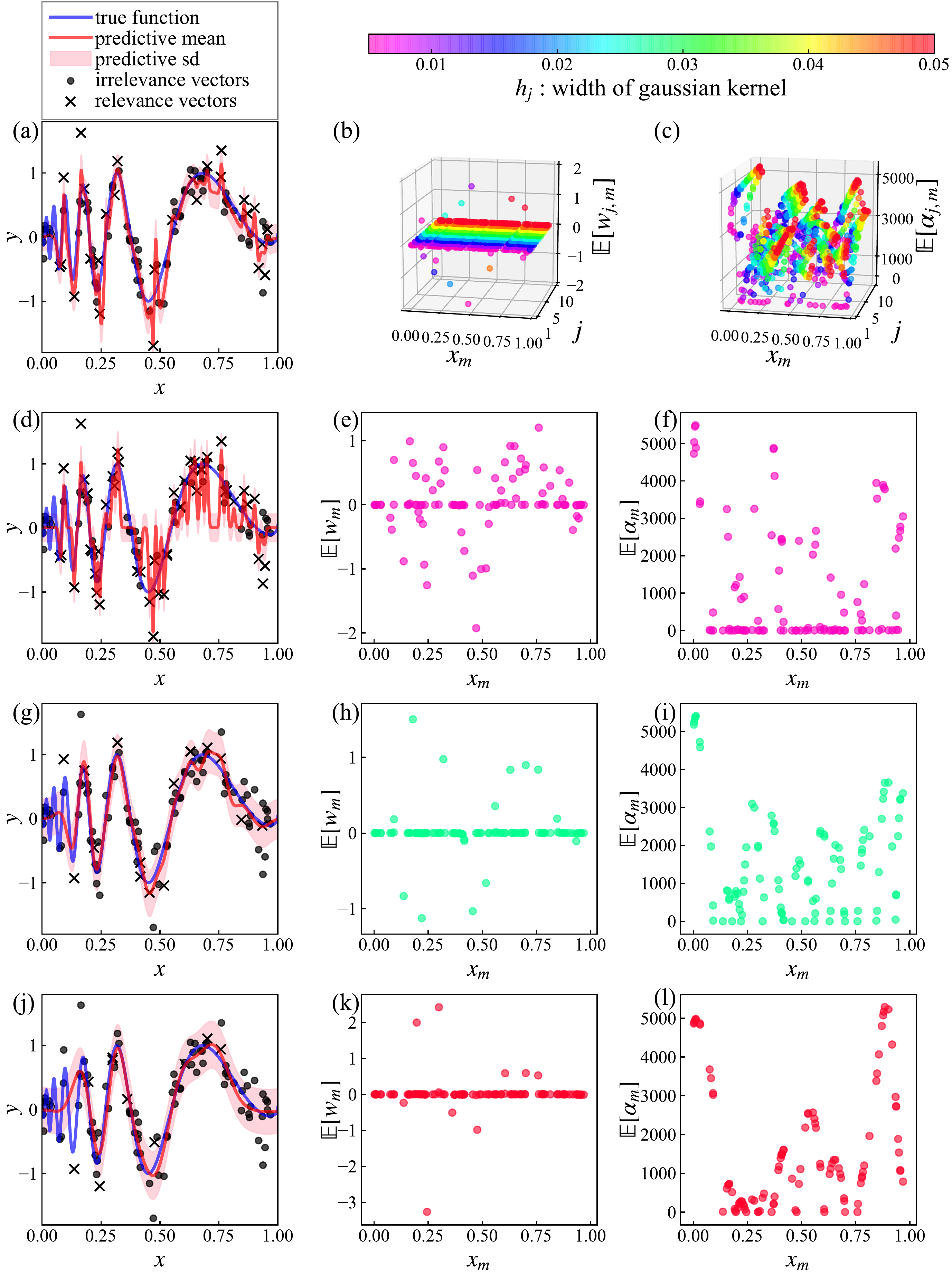}
    \caption{\label{regwalpha_VmVmono_DOPPLER} 
    (a) An one example of estimated regression model in $100$ Monte Carlo trials for DOPPLER data with $(N,\sigma)=(100,0.3)$ 
      using MK-VRVM with $\mathrm{Gam}(\bm{\alpha}|a=10^{-6},b=10^{-6})$. 
      (b) and (c) corresponding $\mathbb{E}{[\bm{w}]}$ and $\mathbb{E}{[\bm{\alpha}]}$.
      (d), (e), and (f) Those for SK-VRVM with $\mathrm{Gam}(\bm{\alpha}|a=10^{-6},b=10^{-6})$ in the case of $h=0.005$. 
      (g), (h), and (i) Those in the case of $h=0.0275$. 
            (j), (k), and (l) Those in the case of $h=0.05$.}
      \end{center}
      \end{figure}
There are primarily two points which we have to explain in terms of predictive accuracy. 
First, proposed regression methods is compared with 
other regression methods. 
Second, we compare $\mathrm{EPIC}_{\gamma}$ with other model selection criteria. 
In both comparisons, our attention should be paid to the results of 
$\mathrm{EPIC}_{\gamma}$ corrected by $Bias_{\mathrm{GIC}}$. 
This is because it can be practical even if the true distribution is not clear. 
However, the following interpretations are the same regardless of 
whether bias correction term in $\mathrm{EPIC}_{\gamma}$ is $Bias_{\mathrm{true}}$, $Bias_{\mathrm{plug}}$, or $Bias_{\mathrm{GIC}}$. 
The interpretations are also regardless of 
whether $df$ of ${}_P\mathrm{C}_{df}$ in $\mathrm{EPIC}_{\gamma}$ is $RVs$ or $\mathrm{Tr}\bm{H}$. 

As shown in Table \ref{tab_BUMPSN100std03}, 
we observe that the MSE and PSE obtained by the MK-VRVM with $\mathrm{InGam}(\bm{\alpha}|a=10^{-6},b)$ 
whose $b$ is selected by $\mathrm{EPIC}_\gamma$ is smaller than those of the other regression methods. 
The pictures of regression could be confirmed through Figs. \ref{regwalpha_EPIC_BUMPS} and \ref{regwalpha_VmVmono_BUMPS}. 
The MK-VRVM with gamma hyperprior involves many relevance vectors and cause over-fitting (see Table \ref{tab_BUMPSN100std03} and 
Fig. \ref{regwalpha_VmVmono_BUMPS}(a))
The corresponding $\mathbb{E}{[\bm{w}]}$ and $\mathbb{E}{[\bm{\alpha}]}$ 
are visually confirmed through Figs. \ref{regwalpha_VmVmono_BUMPS}(b)(c) together with widths of Gaussian kernels. 
In the non-zero weights, the Gaussian kernels with narrow widths seem to be selected, which contributes to over-fitting. 
Conversely, MK-VRVM with $\mathrm{InGam}(\bm{\alpha}|a=10^{-6},b)$ has the moderate number of relevance vectors 
(see Table \ref{tab_BUMPSN100std03}). 
As shown in Figs. \ref{regwalpha_EPIC_BUMPS}(a)(b)(c), 
the necessary and sufficient widths and number of Gaussian kernels are selected to follow the non-homogeneous structure. 
In consequence, the regression model with good predictive accuracy is obtained. 
Figs. \ref{regwalpha_VmVmono_BUMPS}(g)-(l) show 
that effective regression cannot be performed with only single Gaussian kernel with $h=0.0275$ or $0.05$. 
The SK-VRVM of $h=0.005$ with gamma hyperprior succeeds 
in the less smooth region, 
 but causes over-fitting in the smooth region (Figs. \ref{regwalpha_VmVmono_BUMPS}(d)-(f)). 
Here we mention that SK-VRVM with $\mathrm{InGam}(\bm{\alpha}|a=10^{-6},b)$ which is not used in our study 
might perform well when $b$ is selected by $\mathrm{EPIC}_\gamma$. 
Although the  MK-RVM and SK-RVM have a slightly better predictive accuracy than MK-VRVM and SK-VRVM with gamma hyperprior, 
they does not exceed the predictive accuracy of proposed method. 

Next, $\mathrm{EPIC}_\gamma$ is compared with other model selection criteria. 
As shown in Table \ref{tab_BUMPSN100std03}, 
$\mathrm{EPIC}_\gamma$ has the good predictive accuracy than $\mathrm{CV}$, $\mathrm{GCV}$, and $\mathrm{PIC}$. 
Ref. \cite{chen2008extended} pointed out that 
the $\mathrm{CV}$ and the $\mathrm{GCV}$ tend to select a model with many spurious covariates. 
In the terminology of RVM, the covariates correspond to the relevance vectors. 
This property is confirmed in our results together with traditional $\mathrm{PIC}$. 
The $\mathrm{CV}$, $\mathrm{GCV}$, and $\mathrm{PIC}$ seem to select the scale parameter $b$ close to zero (see Table \ref{tab_BUMPSN100std03}). 
We have already shown that $\mathrm{InGam}(\bm{\alpha}|a\sim 0,b\sim 0)$ 
often increase $df$ and cause over-fitting in Figs. \ref{df_b_BUMPS}-\ref{regwalpha_b_DOPPLER}. 
The $\mathrm{EPIC}_\gamma$, on the other hand, selects moderate value of $b$ 
so that the model contains the necessary and sufficient number of relevance vectors. 

The above interpretations are also applied to the results of DOPPLER data, i.e., 
Table \ref{tab_DOPPLERN100std03}, Figs.\ \eqref{regwalpha_EPIC_DOPPLER} and \eqref{regwalpha_VmVmono_DOPPLER}. 
The effect of multiple kernel method is clearly visible in Fig. \ref{regwalpha_EPIC_DOPPLER} than in Fig. \ref{regwalpha_EPIC_BUMPS}. 
The Gaussian kernels with wide width are selected in the smooth region of $x>0$, 
while those with the narrow width are used in the less smooth region of $x<0$ (see Fig. \ref{regwalpha_EPIC_DOPPLER}(b)). 
\subsection{Comparison of Bias Correction of Log-Likelihood\label{sec:5-3}}
This subsection compares the bias correction terms of $\mathrm{EPIC}_\gamma$. 
For this comparison, our attention should be paid to scale parameter $b$ selected by $\mathrm{PIC}$ for simplicity. 
As shown in Tables \ref{tab_BUMPSN100std03} and \ref{tab_DOPPLERN100std03}, 
PICs corrected by $\mathrm{Bias}_{\mathrm{true}}$ and $\mathrm{Bias}_{\mathrm{GIC}}$ 
do not select the same scale parameter $b$. 
We do not emphasize this difference, 
because $\mathrm{EPIC}_\gamma$ corrected by $\mathrm{Bias}_{\mathrm{GIC}}$ has second-order accuracy like 
AIC, TIC, and GIC \cite{konishi2008information}. 
This means that correcting with $\mathrm{Bias}_{\mathrm{GIC}}$ is not a complete correction, 
and the same $b$ is not always selected. 
There is another possibility that the asymptotic theory 
used in the derivation of $\mathrm{Bias}_{\mathrm{GIC}}$ does not hold. 
In the multiple kernel method, the number of parameters in model is $P=1+NJ$. 
When $N$ is increased, $P$ is further increased with respect to $N$. 
In this case, $\mathrm{Bias}_{\mathrm{GIC}}$ might be no longer valid as bias correction of log-likelihood. 
To overcome this situation, we need information criteria that can be used regardless of whether the asymptotic theory holds. 
For example, the WAIC \cite{watanabe2010asymptotic} and WBIC \cite{watanabe2013widely}, 
which have been known to be valid in such a situation, are representative candidates. 
Solving this problem is allocated for future works. 
\subsection{Values of $\gamma$ and Selected Model\label{sec:5-4}}
As mentioned earlier, $\gamma$ is not usually selected objectively. 
We specify 11 $\gamma$ and scale parameter $b$ is selected by $\mathrm{EPIC}_\gamma$. 
In this case, it is important to investigate the relation between $\gamma$ and the selected models 
because this is also discussed in model selection by $\mathrm{EBIC}_\gamma$ \cite{chen2008extended,foygel2010extended,chen2012extended}.
For accurate investigation, we focus on model selection using $\mathrm{EPIC}_\gamma$ corrected by $Bias_{\mathrm{true}}$. 
The following interpretations are the same regardless of 
whether $df$ of ${}_P\mathrm{C}_{df}$ in $\mathrm{EPIC}_{\gamma}$ is $RVs$ or $\mathrm{Tr}\bm{H}$. 
The selected scale parameter $b$, $RVs$, MSE, and PSE against $\gamma$ are plotted 
in Figs. \ref{box_BUMPSN100std03} for the BUMPS and \ref{box_DOPPLERN100std03} for the DOPPLER. 
\begin{figure}[H]
  \begin{tabular}{cc}
    \hspace{-15pt}
    \begin{minipage}[t]{0.5\hsize}
      \begin{center}
      \includegraphics[clip,scale=0.55]{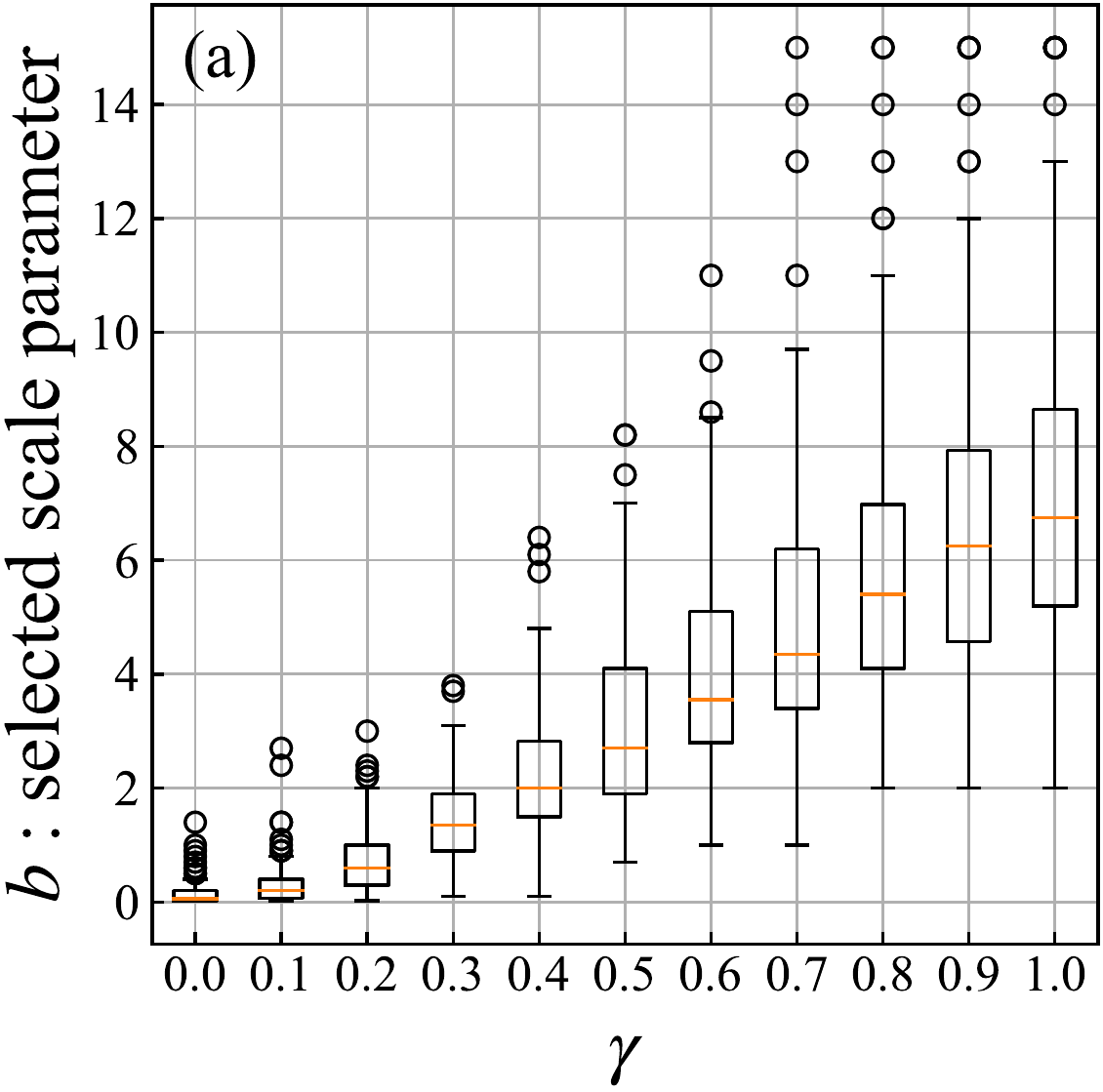}
      \end{center}
    \end{minipage} &
    \hspace{-10pt}
    \begin{minipage}[t]{0.5\hsize}
      \begin{center}
      \includegraphics[clip,scale=0.55]{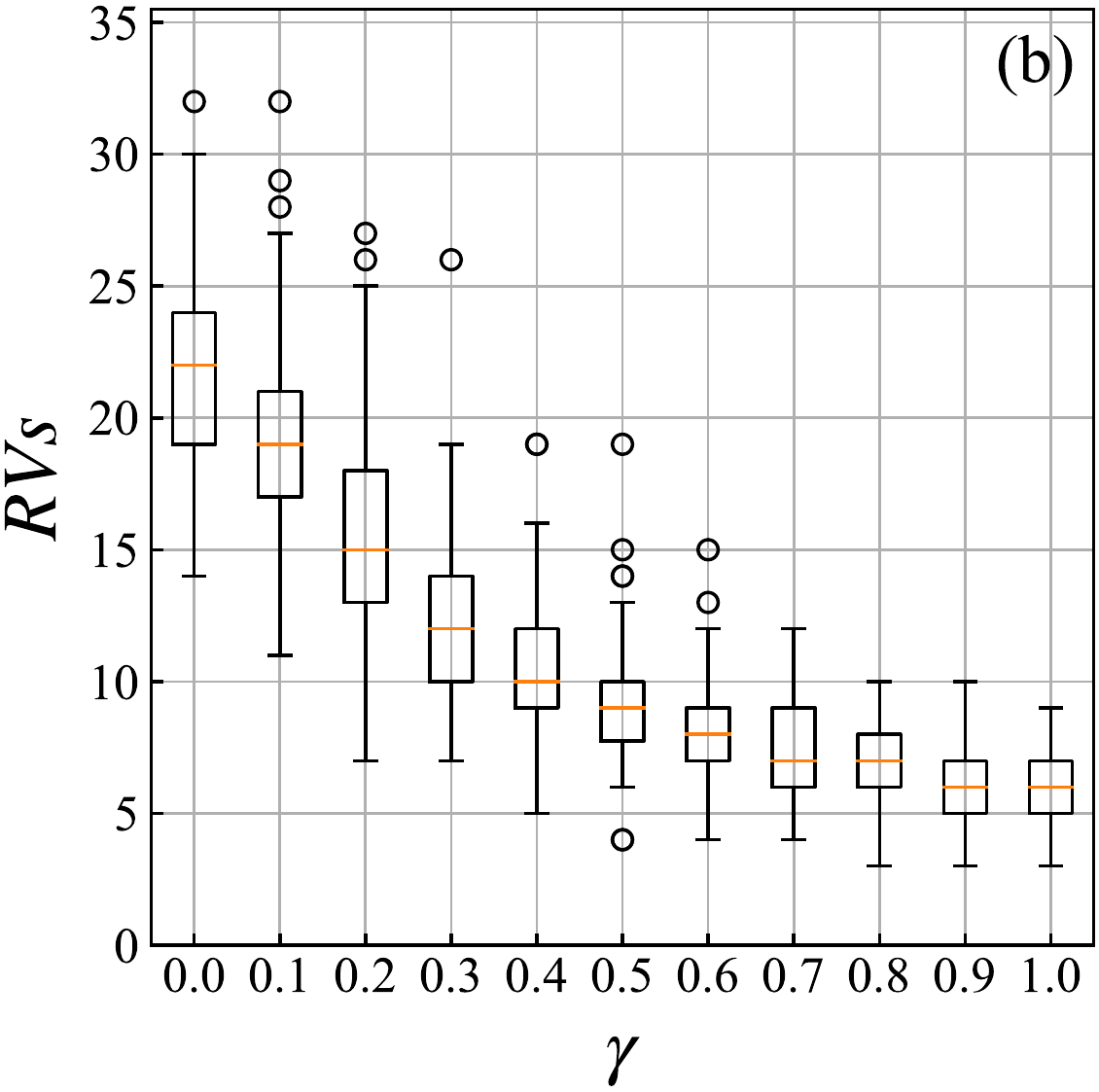}
    \end{center}
    \end{minipage}\\
    \hspace{-15pt}
    \begin{minipage}[t]{0.5\hsize}
      \begin{center}
      \includegraphics[clip,scale=0.55]{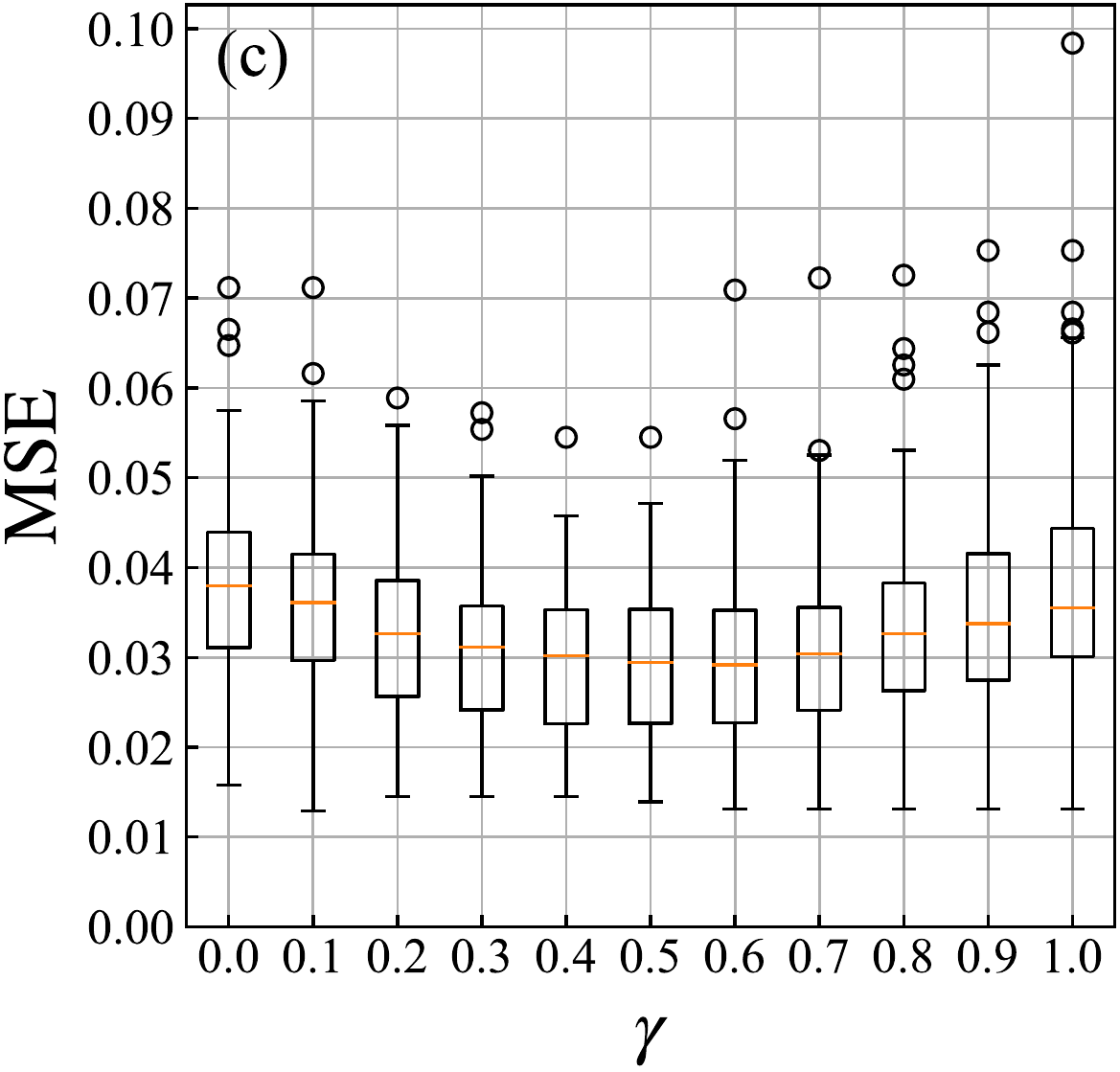}
    \end{center}
    \end{minipage} &
    \hspace{-10pt}
    \begin{minipage}[t]{0.5\hsize}
      \begin{center}
      \includegraphics[clip,scale=0.55]{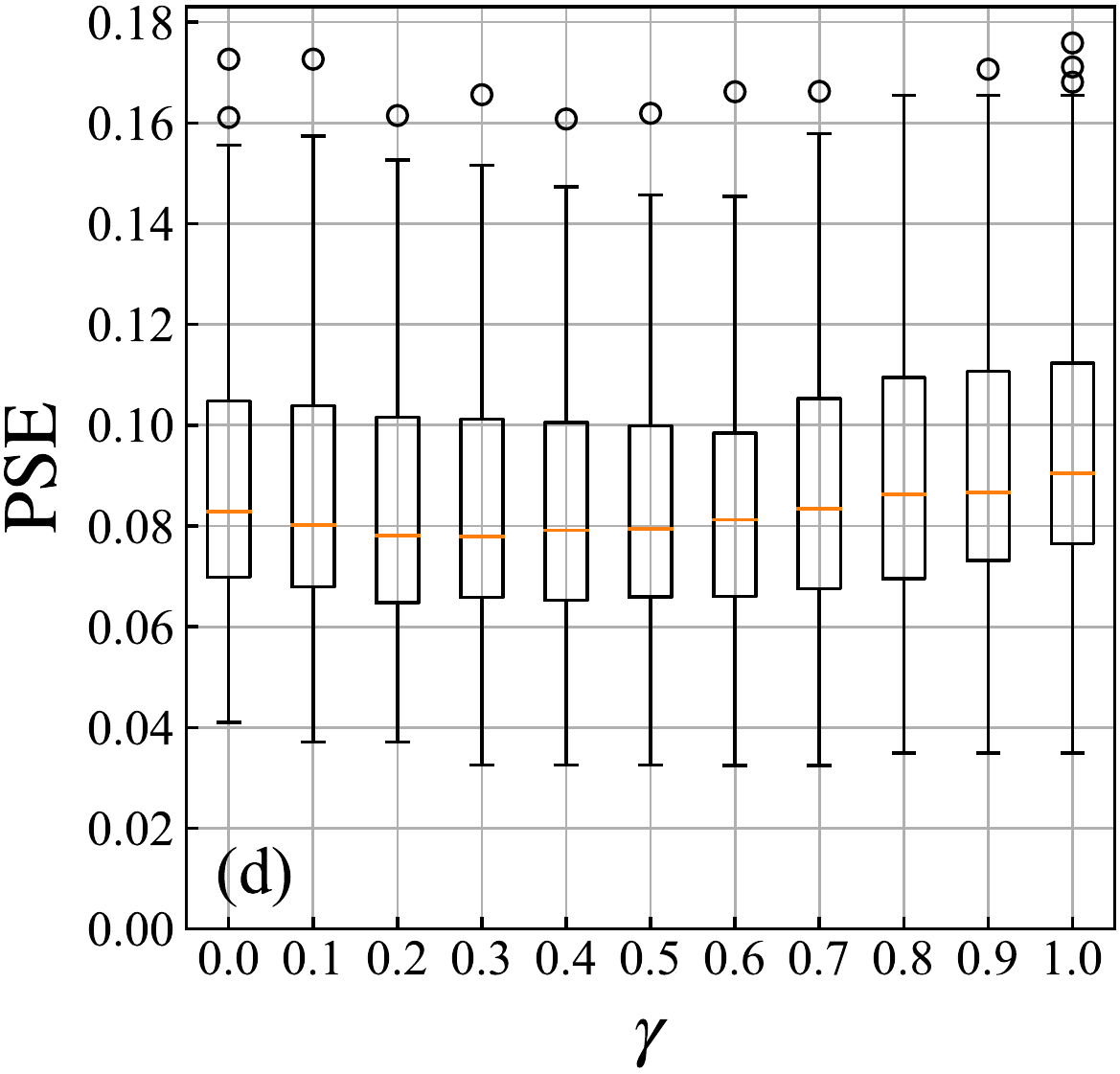}
    \end{center}
    \end{minipage}
  \end{tabular}
  \caption{\label{box_BUMPSN100std03} 
  The box plots of $RVs$, MSE, PSE, and scale parameter $b$ selected by $\mathrm{EPIC}_\gamma$ 
  with $Bias_{\mathrm{true}}$ and ${}_P\mathrm{C}_{df=\mathrm{Tr}\bm{H}}$. 
  These are obtained through 100 Monte Carlo trials for regression to the BUMPS data with $(N=100,\sigma=0.3)$. 
   (a) The box plot of scale parameter $b$. (b) That of $RVs$. 
   (c) That of $\mathrm{MSE}$. (d) That of $\mathrm{PSE}$.
   }
\end{figure}
\begin{figure}[H]
  \begin{tabular}{cc}
    \hspace{-15pt}
    \begin{minipage}[t]{0.5\hsize}
      \begin{center}
      \includegraphics[clip,scale=0.55]{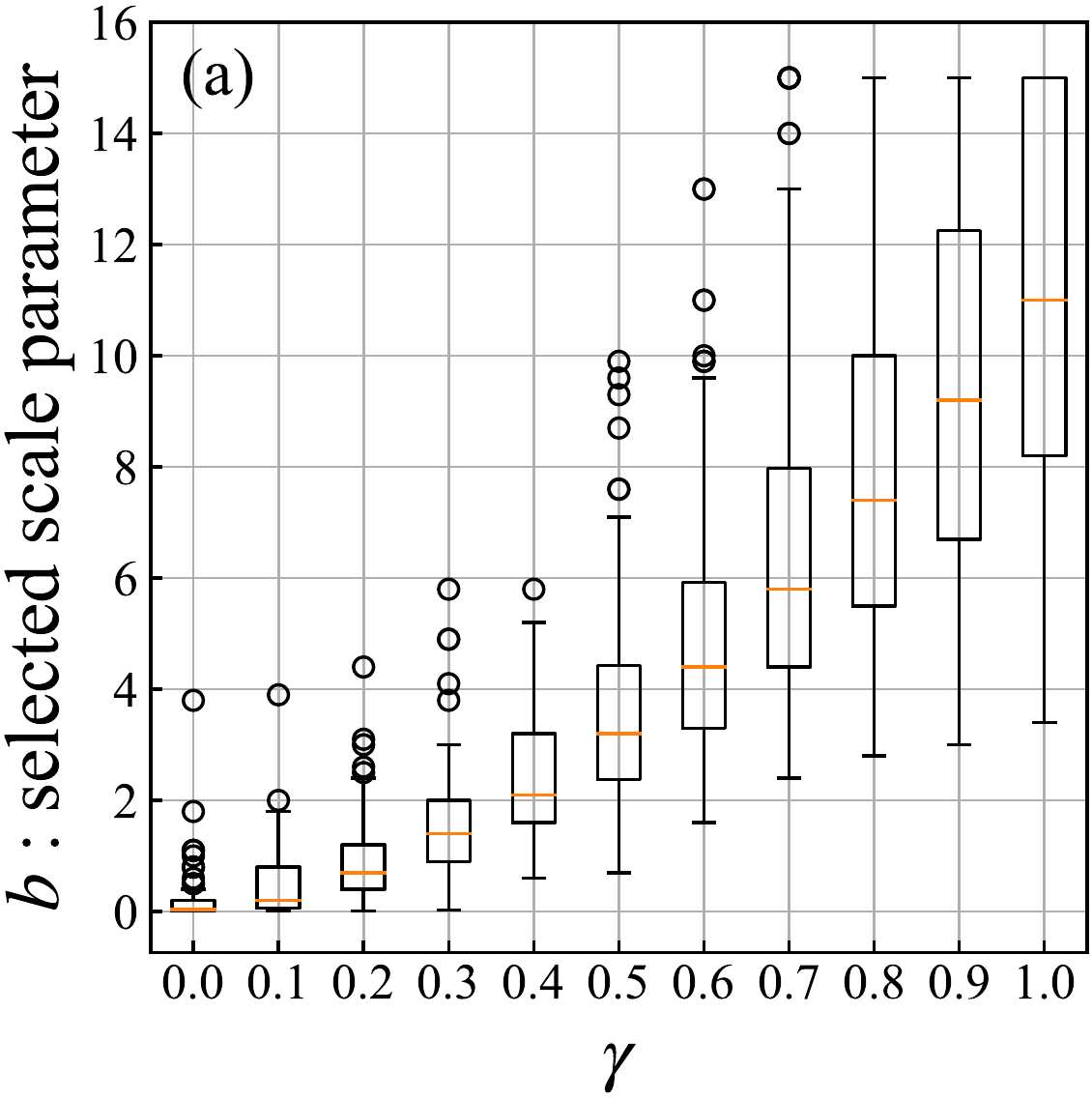}
    \end{center}
    \end{minipage} &
    \hspace{-10pt}
    \begin{minipage}[t]{0.5\hsize}
      \begin{center}
      \includegraphics[clip,scale=0.55]{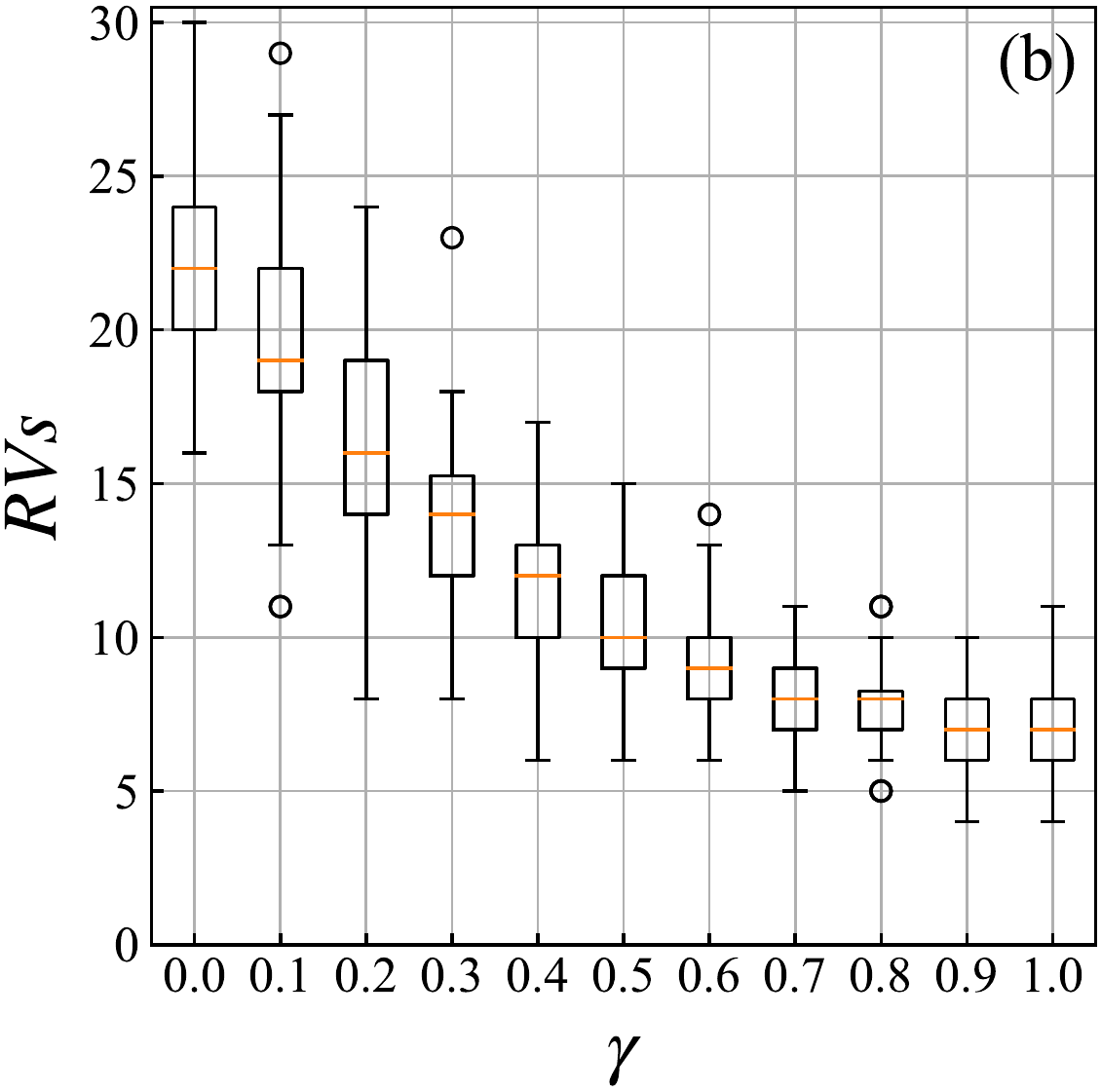}
    \end{center}
    \end{minipage}\\
    \hspace{-15pt}
    \begin{minipage}[t]{0.5\hsize}
      \begin{center}
      \includegraphics[clip,scale=0.55]{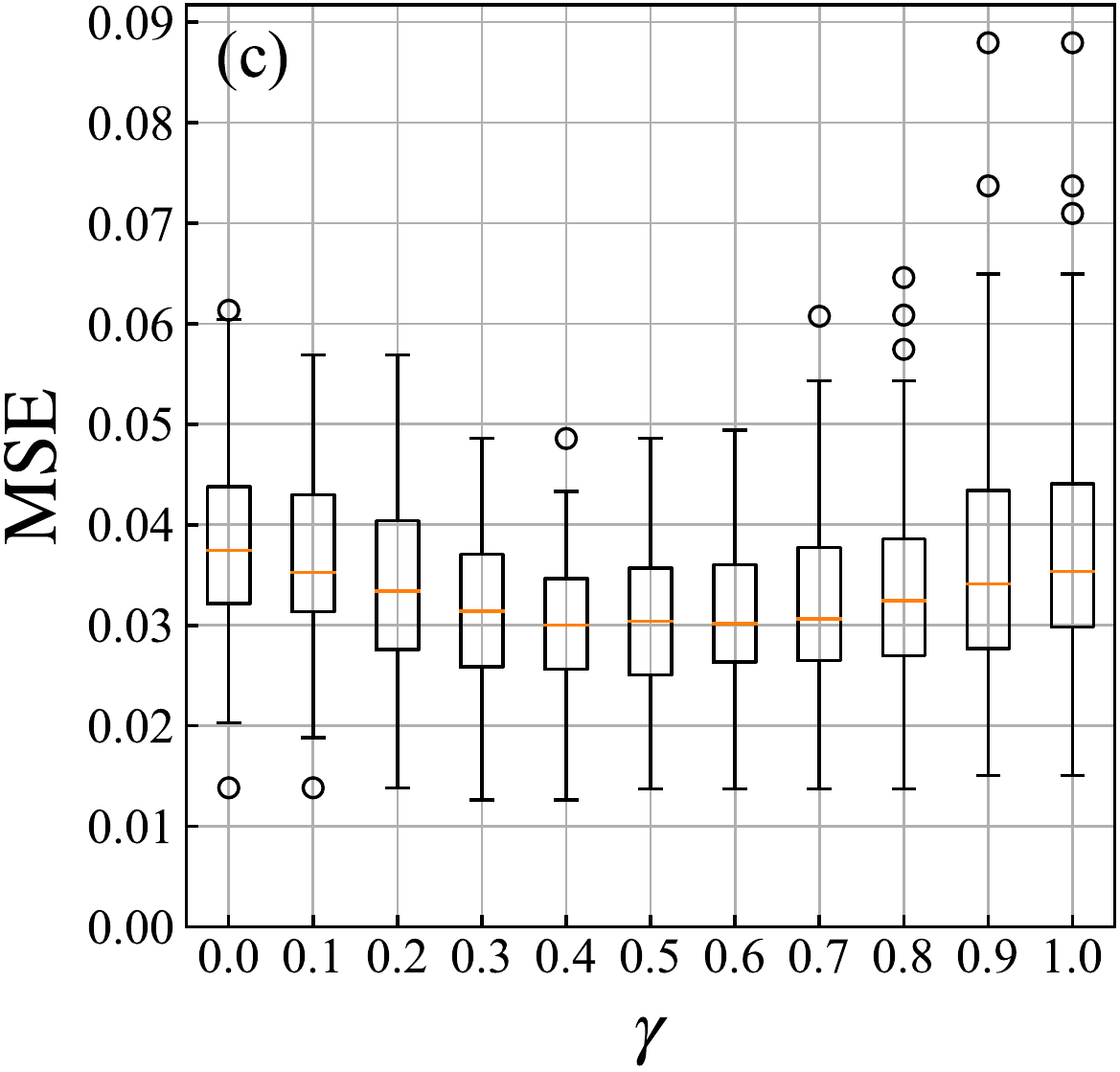}
    \end{center}
    \end{minipage} &
    \hspace{-10pt}
    \begin{minipage}[t]{0.5\hsize}
      \begin{center}
      \includegraphics[clip,scale=0.55]{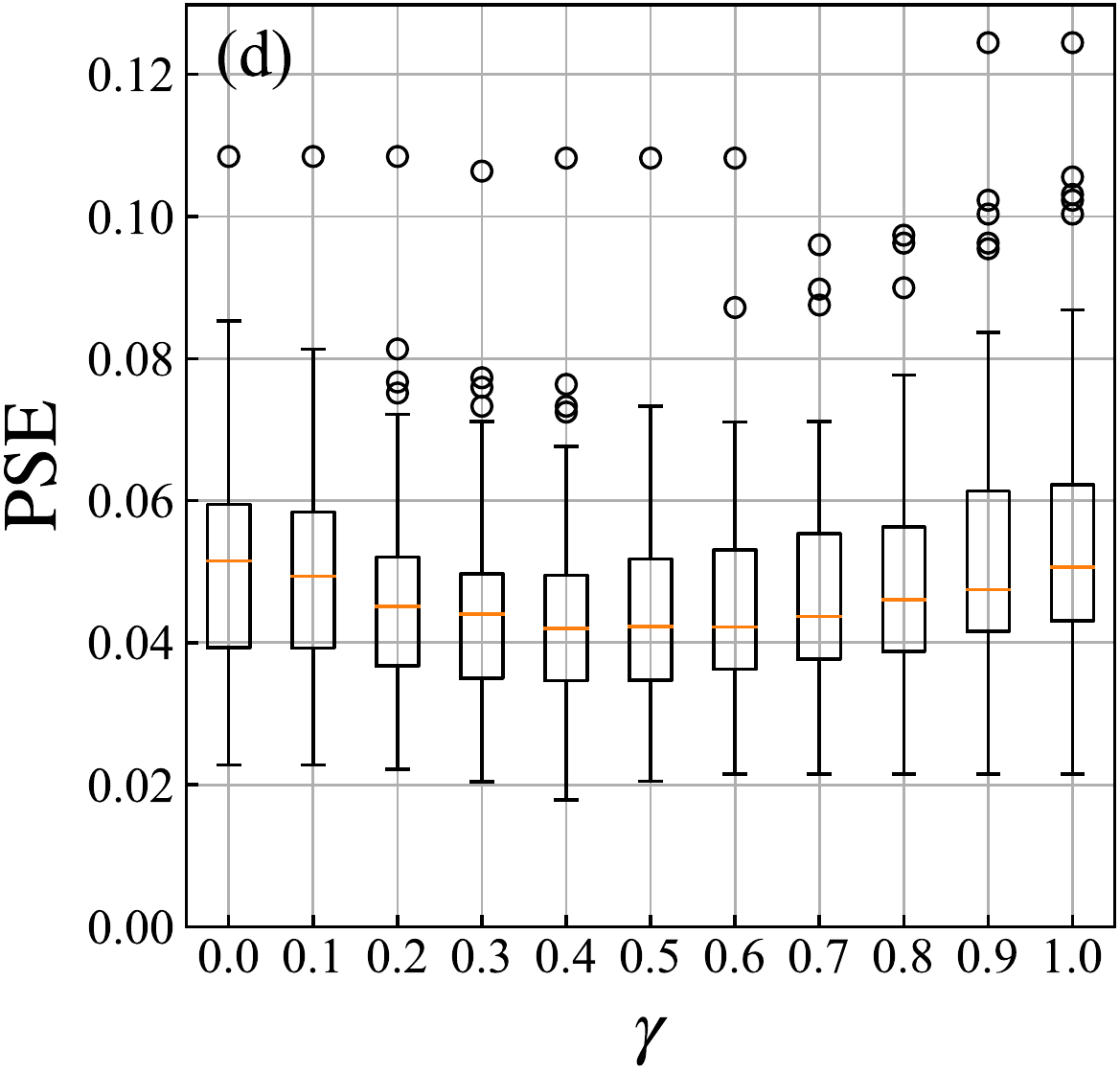}
    \end{center}
    \end{minipage}
  \end{tabular}
  \caption{\label{box_DOPPLERN100std03} 
  The box plots of $RVs$, MSE, PSE, and scale parameter $b$ selected by $\mathrm{EPIC}_\gamma$ 
  with $Bias_{\mathrm{true}}$ and ${}_P\mathrm{C}_{df=\mathrm{Tr}\bm{H}}$. 
  These are obtained through 100 Monte Carlo trials for regression to the DOPPLER data with $(N=100,\sigma=0.3)$. 
   (a) The box plot of scale parameter $b$. (b) That of $RVs$. 
   (c) That of $\mathrm{MSE}$. (d) That of $\mathrm{PSE}$. 
  }
\end{figure}
The characteristic property of $\mathrm{EBIC}_\gamma$ is trade-off relationship 
between the number of covariates included in estimated model and the values of $\gamma$ \cite{chen2008extended}. 
In other words, $\mathrm{EBIC}_\gamma$ with $\gamma$ close to zero is less restrictive on the number of covariates and 
it with $\gamma$ close to one is too restrictive on them. 
This property is confirmed in our results. 
As shown in Figs. \ref{box_BUMPSN100std03}(a) and \ref{box_DOPPLERN100std03}(a), 
the small values of $b$ are selected by $\mathrm{EPIC}_\gamma$ with $\gamma$ close to zero. 
On the other hand, the large values of $b$ are chosen when $\gamma$ increases towards one. 
We say again following; the model becomes sparse as $b$ increases (see Figs. \ref{df_b_BUMPS}-\ref{regwalpha_b_DOPPLER}).
This matter is clearly shown in Figs. \ref{box_BUMPSN100std03}(b) and \ref{box_DOPPLERN100std03}(b). 
In consequence, the trade-off relationship of $\mathrm{EPIC}_\gamma$ may be 
consistent with that of $\mathrm{EBIC}_\gamma$. 

Figs. \ref{box_BUMPSN100std03}(c)(d) and \ref{box_DOPPLERN100std03}(c)(d) tell us that 
$\mathrm{EBIC}_\gamma$ with $\gamma$ close to zero or one does not select a model with good prediction accuracy. 
The former seems to be due to include the relevance vectors which are unnecessary for following the true function, i.e., over-fitting. 
The reason for the latter seems to be restricts even the relevance vectors contributing to the regression for a true function, i.e., under-fitting. 
It is considered that the optimal value of $\gamma$ seems to be between zero and one. 
\section{Concluding Remarks\label{sec:6}}
In this study, the inverse gamma hyperprior with the shape parameter close to zero and the scale parameter not necessary to zero 
has been introduced as hyperprior over noise precision of ARD prior in the relevance vector machine. 
This hyperprior has been non-informative when the scale parameter $b$ approaches zero, while 
weakly informative hyperprior in terms of enhancing sparsity when $b$ increases from zero (see Fig. \ref{comp_priors}). 
The effect of this hyperprior has been confirmed through the regression to non-homogeneous data, applying the multiple kernel method to the variational relevance machine (MK-VRVM). 
The degrees of freedom have decreased with increasing $b$ from zero (see Figs. \ref{df_b_BUMPS}-\ref{regwalpha_b_DOPPLER}), 
which are consistent with Fig. \ref{comp_priors}. 
We have shown that the traditional predictive information criterion (PIC) 
is also obtained from maximizing the posterior probability of the model similar to derivation of BIC. 
In this process, extended PIC (EPIC) is proposed 
by assuming the prior probability of the model not to be uniform. 
The MK-VRVM with inverse gamma hyperprior has performed well in terms of predictive accuracy by 
selecting scale parameter with $\mathrm{EPIC}$ (see Tables \ref{tab_BUMPSN100std03} and \ref{tab_DOPPLERN100std03}). 
The nature of trade-off in $\mathrm{EPIC}$ may have been consistent with that in $\mathrm{EBIC}$. 
\section*{Acknowledgments}
This research was supported by JSPS KAKENHI Grant Number JP19K11854. 
\appendix 
\section{Sequential Algorithm for VRVM with Inverse Gamma Hyperprior\label{sequentalalgo}}
We discuss an application of fast sequential algorithm to VRVM when inverse gamma hyperprior is used. 
The FMLM (Fast Marginal Likelihood Maximization) was developed as such a algorithm in RVM with the second type maximum likelihood. 
In VRVM using variational Bayes, a counterpart of the FMLM was also reported in the case of gamma hyperprior over $\bm{\alpha}$ \cite{shutin2011fast1,shutin2011fast2}. 
This was called FV-SBL (Fast Variational Sparse Bayes Learning). 

We first explain the FV-SBL for VRVM with gamma hyperprior \cite{shutin2011fast1,shutin2011fast2}. 
The stationary points of $\tilde{b}_m$ or $\mathbb{E}{[\alpha_m]}$ and their divergence conditions 
are necessary. 
These works are achieved by solving the equations derived from Eqs.\ \eqref{varw} and \eqref{bmtil}. 
The equation for $m$-th component is given as 
\begin{equation}
  2\left(\tilde{b}_m-b\right)=(\omega_m^2+\varsigma_m)-\frac{\varsigma_m^2+2\varsigma_m\omega_m^2}{\frac{\tilde{b}_m}{a+1/2}+\varsigma_m}+\frac{\varsigma_m^2\omega_m^2}{\left(\frac{\tilde{b}_m}{a+1/2}+\varsigma_m\right)^2}, \label{stpointGam}
\end{equation}
where $\varsigma_m$ and $\omega_m^2$ are given by
\begin{eqnarray}
  \varsigma_m&=&\bm{e}_m^{\mathrm{T}}\tilde{\bm{\Sigma}}_{\overline{m}}\bm{e}_m, \\
  \omega_m^2&=&\mathbb{E}{[\beta]}^2\bm{e}_m^{\mathrm{T}}\tilde{\bm{\Sigma}}_{\overline{m}}\bm{\Phi}^{\mathrm{T}}\bm{y}\bm{y}^{\mathrm{T}}\bm{\Phi}\tilde{\bm{\Sigma}}_{\overline{m}}\bm{e}_m.
\end{eqnarray}
Here 
\begin{equation}
\tilde{\bm{\Sigma}}_{\overline{m}}=\left\{\sum_{k\neq m}\mathbb{E}{[\alpha_k]} \bm{e}_k\bm{e}_k^{\mathrm{T}}+\mathbb{E}{[\beta]}\bm{\Phi}^{\mathrm{T}}\bm{\Phi}\right\}^{-1},
\end{equation}
where $\bm{e}_m=$ is a natural basis for $m$-th coordinate. 
Although Refs. \cite{shutin2011fast1,shutin2011fast2} considered the stationary point of $\mathbb{E}{[\alpha_m]}$, 
it is equivalent to considering that of $\tilde{b}_m$ shown as Eq.\ \eqref{stpointGam}. 
Eq.\ \eqref{stpointGam} could be solved analytically. 
Therefore we could obtain the stationary point of $\tilde{b}_m$ and its divergence condition. 
In consequence, the FV-SBL could be applied when the gamma hyperprior is used. 

Next, we try to apply the FV-SBL to VRVM with inverse gamma hyperprior. 
Unfortunately, this turn out to be difficult. 
The reason is that the stationary points of $\tilde{a}_m$ or $\mathbb{E}{[\alpha_m]}$ could not be solved analytically. 
Here the equation of $\tilde{a}_m$ which we have to solve is 
\begin{equation}
  \tilde{a}_m=(\omega_m^2+\varsigma_m)-\frac{\varsigma_m^2+2\varsigma_m\omega_m^2}{\sqrt{\frac{\tilde{a}_m}{2b}}\frac{K_{\tilde{p}_m}\left(\sqrt{2b\tilde{a}_m}\right)}{K_{\tilde{p}_m+1}\left(\sqrt{2b\tilde{a}_m}\right)}+\varsigma_m}+\frac{\varsigma_m^2\omega_m^2}{\left\{\sqrt{\frac{\tilde{a}_m}{2b}}\frac{K_{\tilde{p}_m}\left(\sqrt{2b\tilde{a}_m}\right)}{K_{\tilde{p}_m+1}\left(\sqrt{2b\tilde{a}_m}\right)}+\varsigma_m\right\}^2}, \label{stpointInGam}
\end{equation}
where this equation is derived from Eqs.\ \eqref{varw} and \eqref{InGamamtil}. 
Solving Eq.\ \eqref{stpointInGam} with respect to $\tilde{a}_m$ is not easy, 
because $\tilde{a}_m$ is included in the argument of $K_p(\cdot)$, which is the modified Bessel function of the second kind. 
Even if the analytical solution of $\tilde{a}_m$ cannot be obtained, 
the divergence condition of it may be found. 
That may allow us to perform faster calculation. 
This is because the divergence condition is used 
to prune the kernel functions in the model, i.e., the size of design matrix can be reduced. 
\section{Derivation of Extended Predictive Information Criterion\label{derivationEPIC}}
The traditional PIC was originally obtained by minimizing the KL distance between 
true distribution and predictive distribution \cite{genshiro1997information}. 
In this appendix, it is derived based on maximization of the posterior probability over the models, 
which is the similar derivation as BIC and $\mathrm{EBIC}_\gamma$ \cite{schwarz1978estimating}. 
In the process of this derivation, the extended predictive information criterion introduced in Sec. \ref{sec:4-1} is obtained. 

Let a model space $\mathcal{S}$ be power set of $\left\{1,2,3,\cdots,P\right\}$. We assume that $\mathcal{S}$ is divided as $\mathcal{S}=\bigcup_{j=1}^{P}\mathcal{S}_j$, 
where $\mathcal{S}_j$ is model space consisting of the models with $j$ covariates, i.e., $\forall s \in \mathcal{S}_j$ $|s|=j$. 
The number of elements in a set is denoted as $|\cdot|$.
Note that $\forall s \in \mathcal{S}$ and $\exists j \in\left\{1,2\cdots,P\right\}$ such that $s \in \mathcal{S}_j$. 
Here we make assumptions in terms of distributions of models as following; 
$p(s|\mathcal{S}_j)=1/|\mathcal{S}_j|$ and $p(\mathcal{S}_j)\propto |\mathcal{S}_j|^{1-\gamma}$, 
where $|\mathcal{S}_j|={}_{P}\mathrm{C}_{j}$ and $\gamma \in [0,1]$. 
In these assumptions, the prior over model $s$ is $p(s)\propto p(s|\mathcal{S}_j)p(\mathcal{S}_j)\propto |\mathcal{S}_j|^{-\gamma}$. 
Let $\forall s \in \mathcal{S}$ specify a parametric distribution $p(\bm{x}|\bm{\theta}_s)$, 
where $\bm{\theta}_s \in \Theta_s \subset \mathbb{R}^{P}$ and $\bm{\theta}_s$ is a $P$-dimensional parameter vector with those components outside $s$ being set to zero. 
Observations $\bm{x} = \left(x_1,\cdots,x_N \right)$ are generated from an unknown probability distribution $G(\bm{x})$ whose
density function is $g(\bm{x})$. 
We denote a predictive distribution given a model $s$ and future data $\bm{z}$ as $h(\bm{z}|\bm{x},s)$. 
\begin{theorem}
  If $\gamma=0$, 
  maximizing $\mathbb{E}_{G(\bm{z})}{[\ln h(\bm{z}|\bm{x},s)]}$ with respect to model $s$ 
    is equivalent to that of $\mathbb{E}_{G(\bm{z})}{[\ln p(s|\bm{Z})]}$. 
  \end{theorem}
\begin{proof} 
A conditional probability of $(\bm{x},s)$ given future data $\bm{z}$ generated independently on observed $\bm{x}$ is 
\begin{equation}
  p(\bm{x},s|\bm{z})\propto h(\bm{z}|\bm{x},s)p(\bm{x},s). \label{xsgivenz} 
\end{equation}
The probability $p(\bm{x},s|\bm{z})$ can be decomposed as $p(\bm{x}|s,\bm{z})p(s|\bm{z})=p(\bm{x}|s)p(s|\bm{z})$ because 
$\bm{x}$ does not depend on future data $\bm{z}$. 
Also decomposing $p(\bm{x},s)$ as $p(\bm{x}|s)p(s)$, Eq.\ \eqref{xsgivenz} reduces to 
\begin{equation}
  p(s|\bm{z})\propto h(\bm{z}|\bm{x},s)p(s). \label{sgivenz}
\end{equation}
The logarithm of Eq.\ \eqref{sgivenz} is 
\begin{equation}
  \ln p(s|\bm{z})=\ln h(\bm{z}|\bm{x},s)+\ln p(s) +\mathrm{const}. \label{logsgivenz}
\end{equation}
In derivation of BIC or EBIC, a counterpart of Eq.\ \eqref{logsgivenz} is $\ln p(s|\bm{x})=\ln p(\bm{x}|s)+\ln p(s) +\mathrm{const}$, 
where $\ln p(\bm{x}|s)$ is marginal likelihood. This could be evaluated directly with Laplace approximation \cite{tierney1986accurate}. 
Conversely, Eq.\ \eqref{logsgivenz} depending on future data $\bm{z}$ could not be calculated directly. 
It is reasonable to evaluate expectation $\mathbb{E}_{G(\bm{z})}{[\cdot]}$, which is given as 
 \begin{equation}
  \mathbb{E}_{G(\bm{z})}{[\ln p(s|\bm{z})]}=\mathbb{E}_{G(\bm{z})}{[\ln h(\bm{z}|\bm{x},s)]}+\ln p(s) +\mathrm{const}. \label{Elogsgivenz}
\end{equation}
In the case of $\gamma=0$, $\ln p(s)=-\gamma \ln |\mathcal{S}_j|+\mathrm{const}=\mathrm{const}$. 
Here we observe that the theorem holds in the case of $\gamma=0$. 
\end{proof}
The maximizing $\mathbb{E}_{G(\bm{z})}{[\ln p(s|\bm{Z})]}$ with respect to model $s$ 
is equivalent to minimizing the $\mathrm{KL}$ distance between $g(\bm{z})$ and $h(\bm{z}|\bm{x},s)$ 
when above theorem holds. 
This is because $\mathrm{KL}\left(g(\bm{z}),h(\bm{z}|\bm{x},s)\right)=\mathbb{E}_{G(\bm{z})}{[\ln g(\bm{z})]}-\mathbb{E}_{G(\bm{z})}{[\ln h(\bm{z}|\bm{x},s)]}$. 
The term $\mathbb{E}_{G(\bm{z})}{[\ln g(\bm{z})]}$ is a constant, depending solely on the $G(\bm{z})$. 

The candidate for an estimator of $\mathbb{E}_{G(\bm{Z})}{[\ln h(\bm{z}|\bm{x},s)]}$ is $\ln h(\bm{x}|\bm{x},s)$. 
This estimator is not unbiased and bias correction of it is necessary. 
In consequence, 
an information criterion, which is maximizing left-hand side of Eq.\ \eqref{Elogsgivenz}, is obtained as 
\begin{equation}
  \mathrm{EPIC}_\gamma = -2\ln h(\bm{x}|\bm{x},s)+2Bias+2\gamma\ln{}_{P}\mathrm{C}_{j}. \label{EPICderivation}
\end{equation}
In the case of $\gamma=0$, Eq.\ \eqref{EPICderivation} is identical to traditional PIC \cite{genshiro1997information}. 
\section{Function of Artificial Data\label{gx}}
The functions used for the numerical evaluation are described. 
These functions are based on Ref. \cite{donoho1995adapting}, while we added slight modifications. 
The function $g(x)$ of BUMPS is 
\begin{equation}
  g(x)=\sum_{i=1}^{7}t_i\left(1+|(x-x_i)/s_i|\right)^{-4},
\end{equation}
where 
\begin{eqnarray*}
  \left\{t_i;i=1,\cdots,7\right\}&=&\left\{1,1.7,1,1.3,1.7,1.4,0.8 \right\}, \\
  \left\{x_i;i=1,\cdots,7\right\}&=&\left\{0.15,0.2,0.3,0.4,0.6,0.75,0.85\right\}, \mathrm{and} \\
  \left\{s_i;i=1,\cdots,7\right\}&=&\left\{0.015,0.015,0.018,0.03,0.03,0.09,0.03\right\}.
\end{eqnarray*}
  That of DOPPLER is 
  \begin{equation}
    g(x)=\sqrt{x(1-x)}\sin\left(2\pi 1.05/(x+0.15) \right).
  \end{equation}
  That of BLOCKS is 
  \begin{equation}
    g(x)=\sum_{i=1}^{7}t_i(1+\mathrm{sgn}(x-x_i))/2,\label{gxBLOCKS}
    \end{equation}
    where 
    \begin{eqnarray*}
    \left\{t_i;i=1,\cdots,7\right\}&=&\left\{1,-1.7,1,-1.3,1.7,-1.4,0.8 \right\} \mathrm{and} \\
    \left\{x_i;i=1,\cdots,7\right\}&=&\left\{0.15,0.2,0.3,0.4,0.6,0.75,0.85\right\}. 
  \end{eqnarray*}
    That of HEAVISINE is 
  \begin{eqnarray}
    g(x)&&=5\sin{\left(4\pi x\right)}+\mathrm{sgn}\left(x-0.1\right)-2\mathrm{sgn}\left(x-0.25\right)\nonumber\\
    &&\qquad -3\mathrm{sgn}\left(x-0.5\right)+4\mathrm{sgn}\left(x-0.75\right)+\mathrm{sgn}\left(x-0.9\right).
  \end{eqnarray}
    Using $g_{\mathrm{min}} = \min \left\{g(x);x\in[0,1]\right\}$ and $g_{\mathrm{max}} = \max \left\{g(x);x\in[0,1]\right\}$, 
    let us defined $C_1=(g_{\mathrm{max}}-g_{\mathrm{min}})/2$ and $C_2=(g_{\mathrm{max}}+g_{\mathrm{min}})/2$. 
    The BUMPS is normalized as follows; $3g(x)/g_{\mathrm{max}}$. 
    The DOPPLER and HEAVISINE are normalized as follows; $(g(x)-C_1)/C_2$. 
    The BLOCKS is not normalized and is used as it is in Eq.\ \eqref{gxBLOCKS}.
\bibliography{RVM} 
\bibliographystyle{unsrt} 
\end{document}